\documentclass{article}

\usepackage[utf8]{inputenc} 
\usepackage[T1]{fontenc}    
\usepackage{url}            
\usepackage{booktabs}       
\usepackage{amsfonts}       
\usepackage{nicefrac}       
\usepackage{microtype}      
\usepackage{array}          
\usepackage{tabularx}       
\usepackage{enumitem}
\usepackage{mathtools}
\usepackage{amsmath}
\usepackage{bm}
\usepackage{amsthm}
\usepackage{graphicx}
\usepackage{subcaption}
\usepackage{multirow}
\usepackage{courier}
\usepackage[binary-units=true]{siunitx}
\usepackage{makecell}
\usepackage{setspace}
\usepackage{todonotes}
\usepackage{xcolor}
\usepackage{svg}
\usepackage{bbm}
\usepackage{xspace}

\usepackage{pifont}
\newcommand{\cmark}{\ding{51}}%
\newcommand{\xmark}{\ding{55}}%

\DeclareSIUnit\px{px}
\pdfoutput=1

\newtheorem{defn}{Definition}
\newtheorem{prop}{Proposition}

\newtheorem{thm}{Theorem}

\newtheorem{definition}{Definition}

\newtheorem*{example*}{Example}
\newtheorem*{prop*}{Proposition}
\newtheorem*{thm*}{Theorem}
\newtheorem*{prv*}{Proof}
\newtheorem*{rmq*}{Remark}

\newcommand{\ie}{\textit{i.e. \xspace}}

\newcommand{\Lip}{\mathrm{Lip}}

\newcommand\norm[1]{\lVert#1\rVert}

\newcommand{\diag}      {\ensuremath{\mathrm{diag}}}

\newcommand{\circulant} {\ensuremath{\mathrm{circ}}}
\newcommand{\bcirc}     {\ensuremath{\mathrm{blkcirc}}}

\newcommand{\cin}{{c_{\text{in}}}}
\newcommand{\cout}{{c_{\text{out}}}}

\def\ddefloop#1{\ifx\ddefloop#1\else\ddef{#1}\expandafter\ddefloop\fi}

\newcommand{\ci}{\ensuremath \mathbf{i}}
\renewcommand{\vector}{\ensuremath{\mathrm{vec}}\xspace}

\def\ddef#1{\expandafter\def\csname #1bb\endcsname{\ensuremath{\mathbb{#1}}}}
\ddefloop ABCDEFGHIJKLMNOPQRSTUVWXYZ\ddefloop

\def\ddef#1{\expandafter\def\csname #1set\endcsname{\ensuremath{\mathcal{#1}}}}
\ddefloop ABCDEFGHIJKLMNOPQRSTUVWXYZ\ddefloop

\def\ddef#1{\expandafter\def\csname #1mat\endcsname{\ensuremath{\mathbf{#1}}}}
\ddefloop ABCDEFGHIJKLMNOPQRSTUVWXYZ\ddefloop

\def\ddef#1{\expandafter\def\csname #1matsf\endcsname{\ensuremath{\mathsf{#1}}}}
\ddefloop ABCDEFGHIJKLMNOPQRSTUVWXYZ\ddefloop

\def\ddef#1{\expandafter\def\csname #1vec\endcsname{\ensuremath{\mathbf{#1}}}}
\ddefloop abcdefghijklmnopqrstuvwxyz\ddefloop

\usepackage[colorlinks,allcolors=blue]{hyperref}

\usepackage[accepted]{icml2023}

\icmltitlerunning{Lipschitz Constant for Convolution}

\begin{document}

\twocolumn[
\icmltitle{Efficient Bound of Lipschitz Constant for Convolutional Layers\\by Gram Iteration}



\icmlsetsymbol{equal}{*}

\begin{icmlauthorlist}
\icmlauthor{Blaise Delattre}{fox,dauph}
\icmlauthor{Quentin Barth\'elemy}{fox}
\icmlauthor{Alexandre Araujo}{nyu}
\icmlauthor{Alexandre Allauzen}{dauph,psl}
\end{icmlauthorlist}

\icmlaffiliation{fox}{FOXSTREAM, Vaulx-en-Velin, France}
\icmlaffiliation{dauph}{Miles Team, LAMSADE, Universit\'e Paris-Dauphine, PSL University, Paris, France}
\icmlaffiliation{nyu}{New York University}
\icmlaffiliation{psl}{ESPCI PSL, Paris, France}

\icmlcorrespondingauthor{Blaise Delattre}{blaise.delattre@dauphine.eu}

\icmlkeywords{deep learning, Lipschitz constant, convolution, spectral norm, power iteration, gram iteration}

\vskip 0.3in
]


\printAffiliationsAndNotice{}  


\begin{abstract}
Since the control of the Lipschitz constant has a great impact on the training stability, generalization, and robustness of neural networks, the estimation of this value is nowadays a real scientific challenge. 
In this paper we introduce a precise, fast, and differentiable upper bound for the spectral norm of convolutional layers using circulant matrix theory and a new alternative to the Power iteration. Called the Gram iteration, our approach exhibits a superlinear convergence.
First, we show through a comprehensive set of experiments that our approach outperforms other state-of-the-art methods in terms of precision, computational cost, and scalability.
Then, it proves highly effective for the Lipschitz regularization of convolutional neural networks, with competitive results against concurrent approaches.
\end{abstract}

\section{Introduction}

Since the landmark result of convolutional neural networks on computer vision tasks \cite{NIPS2012_c399862d}, researchers have tried to develop various methods and techniques to improve the training stability, generalization, and robustness of neural networks.
Recent work has shown that the Lipschitz constant of the network, which is a measure of regularity~\cite{scaman2018lipschitz}, can be tied to all of these properties ~\cite{bartlett2017spectrally,yoshida2017spectral,tsuzuku2018lipschitz}.
Indeed, controlling this constant facilitates training by avoiding exploding gradients, improves generalization on unseen data, and makes networks more robust against adversarial attacks. 
Unfortunately, computing the Lipschitz constant of a neural network is an NP-hard problem~\cite{scaman2018lipschitz}, and while implementation exists~\cite{fazlyab2019efficient,latorre2020lipschitz} their prohibitive computational cost inhibits their use for deep learning training.

To efficiently constrain or regularize the Lipschitz constant, researchers have focused on  individual layers.
In fact, due to the composition property, the product of the Lipschitz of each layer is a (loose) upper bound of the global Lipschitz of the network.
By considering the $\ell_2$-norm, the Lipschitz for each linear transformation reduces to the spectral norm and is equal to the largest singular value of its weight matrix. Efficient iterative methods, called \emph{power method} or \emph{power iteration}, for computing the spectral norm have existed for decades with many variants \cite{golub2000eigenvalue}.

For convolution neural networks, several approaches have been proposed to efficiently estimate this largest singular value of convolutional layers.
\citet{Ryu19Plug} and \citet{farnia2019generalizable} rely on a power iteration approach where the convolution is treated as a linear operator that performs the matrix-vector product. To trade speed for precision, only one iteration is carried out. 
Other works have leveraged the structure of convolution (\ie Toeplitz or circular structure) and the Fourier transform to devise exact but computationally expensive methods~\cite{sedghi2019singular,bibi2019deep} or efficient upper bounds for the largest singular value \cite{singla2021fantastic,araujo2021lipschitz,yi2020asymptotic}.

In this paper, we build upon recent approaches to introduce an efficient procedure for estimating a bound on the Lipschitz constant of convolutional layers. Precise and fast, this bound scales nicely with convolution parameters (the kernel size, channel, and input spatial dimensions). The  procedure is based on \emph{Gram iteration}, a new iterative method to compute the largest singular value with super-linear convergence.
Furthermore, our estimate is guaranteed to be a strict upper bound at each iteration. In contrast,  methods based on power iteration have no such guarantees and can provide a harmful lower bound.   
After presenting the background and theoretical results of our contribution, we perform an extensive set of experiments to demonstrate the superiority of our method over all previous concurrent approaches.
Finally, we demonstrate the usefulness of our method on Lipschitz regularization and show that we can achieve similar accuracy while being more stable.

\begin{table*}[h]
\caption{Qualitative summary of estimations of Lipschitz constant of convolutional layers. Precision and speed qualifications come from experimental results. "deter" abbreviates "deterministic". \citet{Ryu19Plug} uses PI adapted to convolution to compute an estimate arbitrarily precise in a zero padding (a special case of constant padding) setting but suffers from slow convergence. \citet{araujo2021lipschitz} uses the Toeplitz matrix theory to devise a fast bound in spite of precision. Other approaches suppose circular padding for convolution which is a reasonable (comparable performance) and useful assumption to reduce computation complexity thanks to the Fourier transform. \citet{sedghi2019singular} provides exact computation of singular values, however, the method is very slow and does not scale. \citet{singla2021fantastic} proposes a faster using PI method to the detriment of precision. Methods using PI are not guaranteed to produce an upper bound on convolution spectral norm regarding constant or circular padding. 
Our method gathers the best of both worlds: being fast, precise, deterministic, well differentiable, and with a strict upper bound on the spectral norm.
}
\centering
  \vspace{0.1cm}
 \begin{tabular}{lcccccc}
 \toprule
 \textbf{Methods} & \textbf{Precision} & \textbf{Speed} & \textbf{Padding} & \textbf{Upper Bound} & \textbf{Algorithm} & \textbf{Convergence} \\
 \midrule
  {\small \citeauthor{Ryu19Plug,farnia2019generalizable}} & ++ & - & zero & \xmark & PI: iterative, non-deter & linear \\
  {\small \citeauthor{sedghi2019singular,bibi2019deep}} & ++ & - -  & circular & exact & SVD: deter & - \\
  {\small \citeauthor{singla2021fantastic,yi2020asymptotic}} & + & + & circular & \xmark & PI: iterative, non-deter & linear \\
  {\small \citeauthor{araujo2021lipschitz}} & - & ++ & zero & \cmark & close-form, deter & -\\
  \midrule
  \textbf{Ours} & ++ & ++ & circular & \cmark & GI: iterative, deter & superlinear \\
  \bottomrule
\end{tabular}
\label{tab:review}
\end{table*}

\section{Related Work}

The Lipschitz constant for the $\ell_2$-norm of a neural network $f$ is defined as follow for an input space $\mathcal{X}$: 
\begin{equation}
    \label{eq:def_lip}
  \Lip{(f)} = \sup_{\substack{x, x' \in \mathcal{X} \\ x \neq x'}} \frac{\lVert f(x) - f(x') \rVert_2}{\lVert x - x' \rVert_2} \enspace .
\end{equation}
It measures network sensitivity toward an input perturbation $\epsilon$: $\norm{f(x) - f(x+\epsilon)}_2 \leq \Lip(f) ~ \norm{\epsilon}_2$, using Equation~(\ref{eq:def_lip}). If $\ell_2$-norm is a well-spread choice, other approaches use $\ell_\infty$-norm \cite{anil2019sorting,zhang_rethinking_2022}.
In this work, we are interested in estimating and controlling the Lipschitz constant of neural networks which are compositions of layers \ie $f = \phi_l \circ \phi_{l-1} \circ \dots \circ \phi_1$. Computing the Lipschitz constant of $f$ is a hard problem \cite{scaman2018lipschitz} and current implementations to estimate it  directly fail to scale for deep architecture. Some approaches compute the Lipschitz constant using bound propagation techniques in network verification \cite{zhang_recurjac_2019,jordan_exactly_2020,shi_efficiently_2022}.
However, a simple bound on the network Lipschitz constant can be derived by considering  the individual layers:
\begin{align}
  \label{eq:prod_lip_bound}
  \Lip(f) &\leq \prod_{i=1}^l \Lip(\phi_i) \ .
\end{align}
A technique to enforce control on the Lipschitz constant is orthogonalization of layer $\phi_i$ by design \cite{trockman2021orthogonalizing}, by regularization \cite{wang2020orthogonal,huang2020controllable} or by optimizing on manifold \cite{anil2019sorting}. Another approach is to perform spectral normalization \cite{yoshida2017spectral, miyato2018spectral, farnia2019generalizable} or regularization \cite{gouk2021regularisation}.

For dense layers, the problem boils down to matrix spectral norm computation or equivalently computing the largest singular value. Power iteration (PI) is a classical solution that serves as the basis for many others, and it is described in Algorithm~\ref{algo:power_iteration}, where $*$ corresponds to the conjugate transpose operator.
The convergence is linear with a rate of convergence of $\sigma_1 / \sigma_2$. This ratio can be close to $1$ and other variations of PI improve it \cite{lanczos1950iteration,arnoldi1951principle, KUBLANOVSKAYA1962637QRmethod} and more recently \cite{lehoucq1996deflation}.
However, these methods suffer from important drawbacks. First, many iterations are required to ensure full convergence, which can be computationally expensive. Second, the method is not fully deterministic, as it starts from a randomly generated vector. Finally, intermediate iterations are not guaranteed to be a strict upper bound on the spectral norm.
\begin{algorithm}[h]
\caption{: Power\_iteration$(G, N_\text{iter})$}
\label{algo:power_iteration}
\begin{algorithmic}[1]
  \STATE{\textbf{Inputs} matrix: $G$, number of iterations: $N_\text{iter}$}
  \STATE \textbf{Initialization}: draw a random vector $u$ 
  \STATE \textbf{for} $1 \ldots N_\text{iter}$
    \STATE \quad $v \gets Gu / \norm{Gu}_2$ 
    \STATE \quad $u \gets G^* v / \norm{G^* v}_2$ 
  \STATE $\sigma_1 \gets (Gu)^*  v$
  \STATE \textbf{return} $\sigma_1$
\end{algorithmic}
\end{algorithm}

\cite{miyato2018spectral} reshapes convolutional filter $\cout \times \cin \times k \times k$  to a matrix of size $\cout \times \cin \times k^2$ and computes its spectral norm as a loose proxy for convolution one.
Works of \cite{Ryu19Plug, farnia2019generalizable} generalized PI to convolutional layers leveraging transpose convolution operator to produce an exact estimate. However, time cost increases strongly with input spatial size and kernel size, and differentiability is not complete as the gradient is not accumulated in vector $u$ and $v$, as explained in \cite{singla2021fantastic}.

Another line of work exploits the structure of convolution to calculate its spectral norm with reduced computational complexity \cite{sedghi2019singular, bibi2019deep, singla2021fantastic, araujo2021lipschitz}.
In \citet{sedghi2019singular}, circular padding convolutions are represented using properties of doubly-block circulant matrix. Then the Fourier transform can be used to express singular values of convolutions. 
This method leads to the calculation of  $n^2$ spectral norms for $\cout \times \cin$ matrix using SVD. The result is exact but the method is very slow and 
\citet{singla2021fantastic} proposes a method that further reduces computation but at the expense of a decreased precision on the Lipschitz constant.

Table~\ref{tab:review} is a qualitative summary of different methods to estimate the Lipschitz constant of convolutional layers.
\citet{Ryu19Plug} proposes a convolution-adapted PI method  in which  computational time scales with convolution product, it can be costly to derive a very precise spectral norm estimation as convergence is linear. Concerning \citet{araujo2021lipschitz} method, it is fast but only precise for $\cout \times 1 \times k \times k$ or $1 \times \cin \times k \times k$ convolutional filters.
\citet{sedghi2019singular} provides an exact estimate of the Lipschitz constant under circular convolution hypothesis by performing $n^2$ SVD on $\cout \times \cin $ matrices which is very slow and not scalable for large convolution filters. 
Using the same hypothesis, \citet{singla2021fantastic} proposes a bound by taking a minimum of four $\cout k \times \cin k$ matrix spectral norms using PI to compute them, using PI makes the method faster but at the expense of degrading bound precision. 
\citet{yi2020asymptotic} proves that this type of bounds also holds in zero padding case. 
It is important to note that all methods using PI  \cite{Ryu19Plug, singla2021fantastic} produce do not offer upper bound guarantees.
For experiences, we choose the \citet{Ryu19Plug} method with $100$ iterations (to ensure convergence) as a reference value. Indeed, most convolutional layers of CNNs use constant padding. \cite{sedghi2019singular} gives an exact spectral norm for circular padding convolution.
In recent works, researchers have mainly focused on speed but at the expense of precision \cite{singla2021fantastic, araujo2021lipschitz}. Our method gathers both methods' strengths: differentiable, scalable, very fast with superlinear convergence, and with a precision matching \citet{sedghi2019singular} method.

\section{Background on Convolutional Layers}
\label{section:lip_conv_layers}


In this section, we present some properties of convolutional layers with circular padding based on the inherent matrix structure of the convolution operator.

Let $X \in \Rbb^{n \times n}$ be the input, and $K \in \Rbb^{k \times k}$ a convolution kernel. In the following, the kernel will be always considered as padded towards input size, \ie $K \in \Rbb^{n \times n}$.
Convolution product $\star$ can be expressed as a matrix-vector product, where the structure of matrix $W \in \Rbb^{n^2 \times n^2}$ takes into account the weight sharing of convolution \cite{jain1989fundamentals}:
\begin{equation}
  \vector(K \star X) = W \vector(X) \ .
\end{equation}

Using a \emph{circular padding}, $W$ is a doubly-block circulant matrix in which each row block is a circulant shift of the previous one and where each block is a circulant matrix \cite{jain1989fundamentals}.
Denoting $\circulant(\cdot)$ the operator generating a circulant matrix form a vector, $W$ is fully determined by the kernel $K$ as:

\begin{equation*}
  W = \begin{psmallmatrix}
    \circulant(K_1) & \circulant(K_2) & \cdots & \circulant(K_n) \\
    \circulant(K_n) & \circulant(K_1) & \ddots & \vdots \\
    \vdots & \ddots & \circulant(K_1) & \circulant(K_2) \\ 
    \circulant(K_2) & \cdots & \circulant(K_n) & \circulant(K_1) \phantom{\vdots}
  \end{psmallmatrix}.
\end{equation*}

In the following, we denote $W = \bcirc(K_1, \dots, K_n)$.
However the choice of padding for the convolution has a crucial effect; indeed, with zero padding, the structure of the convolution matrix is not doubly-block circulant but \emph{doubly-block Toeplitz}, and the following results with the Fourier diagonalization do not hold.   

Now, based on the doubly-block circulant structure of convolution with circular padding, we present some useful properties. First, we formally introduce the discrete Fourier matrix and some known results. 
\begin{definition}[Fourier Transform] 
\label{def:dft_matrix}
The discrete Fourier transform matrix $U \in \Cbb^{n \times n}$ is defined such that, for $1 \leq u,v \leq n$ :
\begin{equation*}
 U_{u,v} = e^{-\frac{2\pi\ci uv}{n}}
\end{equation*}
and its inverse is defined as $U^{-1} = \frac1n \ U^*$.
\end{definition}

Operating on a vectorized 2D input $X$, 2D discrete Fourier transform is expressed as $\vector(U^\top X U) = (U \otimes U) \vector(X)$, where $\otimes$ is the Kronecker product. We denote $F = U \otimes U$ in the following of the paper.

\begin{thm}[Section 5.5 of~\citet{jain1989fundamentals}]
Let $K \in \Rbb^{n \times n}$ a convolutional kernel and $W \in \Rbb^{n^2 \times n^2}$ be the doubly-block circulant matrix such that $W = \bcirc\left(K_1, \dots, K_n\right)$, then, $W$ can be diagonalized as follows:
\begin{equation}
  W = F \diag(\lambda) F^{-1}
\end{equation}
where $\lambda = F \vector(K)$ are the eigenvalues of $W$.
\end{thm}

In the context of deep learning, convolutions are applied with multiple channels: $\cout$ convolutions are applied on input with $\cin$ channels.
Therefore, the matrix associated with the convolutional filter $K \in \Rbb^{\cout \times \cin \times n \times n}$ becomes a block matrix, where each block is a doubly-block circulant matrix: $W = (W_{i,j})$, for $1 \leq i \leq \cout$ and $1 \leq j \leq \cin$, with $W_{i,j} = \bcirc(K_{i, j, 1}, \dots, K_{i, j, n})$.
As also studied by \citet{sedghi2019singular} and \citet{yi2020asymptotic},
this type of block doubly-block circulant matrix can be \emph{block} diagonalized as follows.

\begin{thm}[Corollary A.1.1. of~\citet{trockman2021orthogonalizing}]
\label{theorem:block_diagonalization}
Let $P_\text{out} \in \Rbb^{\cout n^2 \times \cout n^2}$ and $P_\text{in} \in \Rbb^{\cin n^2 \times \cin n^2}$ be permutation matrices and let $W \in \Rbb^{\cout n ^2 \times \cin n^2}$ be the matrix equivalent of the multi-channel circular convolution defined by filter $K \in \Rbb^{\cout \times \cin \times n \times n}$, then $W$ can be block diagonalized as follows:
\begin{equation}
  W = (I_{\cout} \otimes F) P_\text{out} D P_\text{in}^\top (I_{\cin} \otimes F^{-1})
\end{equation}
where $D$ is a block diagonal matrix with $n^2$ blocks of size $\cout \times \cin$ and where 
\begin{equation*}
  D = P_\text{out}^\top 
   \begin{pmatrix}
     F \vector(K_{1, 1}) & \cdots & F \vector(K_{1, \cin}) \\
     \vdots & & \vdots \\
     F \vector(K_{\cout, 1}) & \cdots & F \vector(K_{\cout, \cin})
   \end{pmatrix} 
   P_\text{in} \ .
\end{equation*}
\end{thm}
Matrices $P_\text{out}$ and $P_\text{in}$ reshaping matrix $D$ into a block diagonal matrix correspond to an alternative vectorization of layer input \cite{sedghi2019singular,yi2020asymptotic} while keeping singular values identical \cite{henderson1981TheVM}.

Based on this result and on the properties of block diagonal matrices, it is easy to compute the largest singular value $\sigma_1$ of $W$ from the block diagonal matrix $D$.
Let $(D_i)_{i = 1, \dots, n^2}$ be the diagonal blocks of the matrix $D$, then:
\begin{equation}
  \label{equation:lip_conv_circu_expression}
  \sigma_1(W) = \sigma_1(D) = \max_i \sigma_1(D_i) \ .
\end{equation}
It requires calculating the spectral norm of each $D_i$, one could use SVD as in  \cite{sedghi2019singular} or PI as in \cite{yi2020asymptotic}, however, these algorithms do not scale up nor are efficient. In the following section, we present a more precise, fast, and scalable approach.



\section{Lipschitz Constant of Convolutional Layers}
\label{subsection:spectral_norm_gram}

\subsection{Gram iteration}
\label{ssec:spectral_norm_gram}

In this part, we derive a new method to estimate an upper bound of the spectral norm, deferring all the proofs to Appendix~\ref{sec:proofs}.
We consider a general matrix $G \in \mathbb{C}^{p \times q}$ and its singular values $\sigma(G)$.

\begin{definition}[Schatten $\rho$-norms \cite{bhatia2013matrix}]
\label{def:schatten_norm}
For $\rho \in [1, +\infty)$, Schatten $\rho$-norms of matrix $G$ are defined as:
\begin{equation*}
s_\rho(G) = \norm{\sigma(G)}_\rho = \left( \sum_{i=1}^{\min\{p, q\}} \sigma_i^\rho(G) \right)^{1/\rho} \ ,
\end{equation*}
where $\rho=2$ gives Frobenius norm $\norm{G}_F$ and $\rho=+\infty$ gives spectral norm $\sigma_1(G)$.
\end{definition}

We define the recurrent sequence $(G^{(t)})$ as, $\forall t \geq 1$:
\begin{align}
   \left\{
    \begin{array}{ll}
    G^{(1)} &= G \\
    G^{(t+1)} &= {G^{(t)}}^* G^{(t)} \ ,
    \end{array}
    \right.
\end{align}
where $G^{(t+1)}$ is the Gram matrix of $G^{(t)}$ \cite{bhatia2013matrix}.

We observe that the $2^t$-th root of the squared Frobenius norm of $G^{(t)}$ is the Schatten $2^t$-norm of $G$:
\begin{equation}
\sqrt[2^t]{\norm{G^{(t)}}^2_F} = \norm{G^{(t)}}_F^{2^{1-t}} = \norm{\sigma(G)}_{2^t} \ ,
\end{equation}
which converges superlinearly (almost quadratically) to $\norm{\sigma(G)}_\infty$, which is equal to the spectral norm of G, \ie $\sigma_1(G)$. Those results are formalized in the following theorem.

\begin{thm}[\textbf{Main result}]
\label{thm:main_result}
Let $G \in \mathbb{C}^{p \times q}$ and define the recurrent sequence  $G^{(t+1)} = {G^{(t)}}^* G^{(t)},$
with $G^{(1)} = G$. Let $(\norm{G^{(t)}}_F^{2^{1-t}})$ be another sequence based on $G^{(t)}$, then, $\forall t \geq 1$, we have the following results:
\begin{itemize}[parsep=0pt,itemsep=0pt,topsep=0pt,leftmargin=10pt]
    \item The sequence is an upper bound on spectral norm:
    \begin{equation*}
        \sigma_1(G) \leq  \norm{G^{(t)}}_F^{2^{1 -t}} \ ;
    \end{equation*}
    \item The sequence converges to the spectral norm:
    \begin{equation*}
        \norm{G^{(t)}}_F^{2^{1 -t}} \underset{t \to +\infty}{\longrightarrow} \sigma_1(G) \ ,
    \end{equation*}
    and this convergence is Q-superlinear of order $2 - \epsilon$, for $\epsilon$ arbitrary small.
\end{itemize}
\end{thm}

We denote the iterative method described in Theorem~\ref{thm:main_result} as \emph{Gram iteration} (GI) due to the recurrence relation: $G^{(t+1)}$ is the Gram matrix of $G^{(t)}$. A pseudo-code is given in Algorithm~\ref{algo:gram_iteration_dense}.
Note that to avoid overflow, matrix $G$ is rescaled at each iteration and scaling factors are cumulated in variable $r$, in order to unscale the result at the end of the method, see Appendix~\ref{section:gram_iteration_rescaling}. Unscaling is crucial as it is required to remain a strict upper bound on the spectral norm at each iteration of the method.
One can note that GI gives a proper norm at each iteration, which is a deterministic upper bound on the spectral norm, and which converges quasi-quadratically. However, GI is not a matrix-free method, contrary to PI and its variants.

\begin{algorithm}[h]
\caption{: Lip\_dense$(G, N_\text{iter})$}
\label{algo:gram_iteration_dense}
\begin{algorithmic}[1]
  \STATE \textbf{Inputs} matrix: $G$, number of iterations: $N_\text{iter}$
  \STATE  $r \gets 0$ \hfill \text{// initialize rescaling}
  \STATE \textbf{for} $1 \ldots N_\text{iter}$
    \STATE  \quad $r \gets 2(r + \log \norm{G}_F)$ \hfill \text{// cumulate rescaling}
    \STATE \quad $G \gets G/\norm{G}_F$ \hfill \text{// rescale to avoid overflow}
    \STATE \quad $G \gets G^* G$ \hfill \text{// Gram iteration}
  \STATE $\sigma_1 \gets \norm{G}_F^{2^{-N_\text{iter}}} \exp{(2^{-N_\text{iter}} r})$
  \STATE \textbf{return} $\sigma_1$
\end{algorithmic}
\end{algorithm}

Given our interest in optimizing within this bound, differentiability is a crucial property. In the following, we study the differentiability of the method and provide an explicit gradient formulation. 

\begin{prop}
\label{prop:gradient_gram_iteration} 
Bound sequence $(\norm{G^{(t)}}_F^{2^{1 -t}})$ is differentiable regarding $G$:
\begin{equation}
\label{eq:gram_iteration_bound_gradient}
     \frac{\partial (\norm{G^{(t)}}_F^{2^{1 -t}})}{\partial G} =\frac{G ( G^* G)^{2^{t-1} -1}}{(\norm{G^{(t)}}_F)^{2(1 - 2^{-t})}} \ .
\end{equation}
\end{prop}
This proposition is key as it enables the implementation of the bound with explicit gradient for optimization, and thus provides an efficient way to backpropagate during training. Note that the gradient in GI is complete, whereas the gradient in PI vectors is not accumulated during iterations.

 
\subsection{Spectral norm of convolutional layers}
\label{ssec:spectral_norm_conv}

GI described in Algorithm~\ref{algo:gram_iteration_dense} can be used directly to estimate the spectral norm of dense layers. For convolutional layers, we use Equation~(\ref{equation:lip_conv_circu_expression}) and apply Theorem~\ref{thm:main_result} on sequences $(D_i^{(t)})$.

\begin{prop}
An upper bound of the spectral norm of a circular convolutional layer is estimated as:
\label{prop:upper_bound_conv}
\begin{equation}
\label{equation:bound_conv_gram}
  \sigma_1(W) \leq \max_{1 \leq i \leq n^2} \norm{D^{(t)}_{i}}_F^{2^{1-t}}  \underset{t \to +\infty}{\longrightarrow} \sigma_1(W) \ .
\end{equation}
\end{prop}

Because $\norm{D_i^{(t)}}_F^{2^{1 -t}} = s_{2^{t}}(D_i)$ is a Schatten norm, it is convex and a maximum over norms establishes a convex function. Hence the upper bound is subdifferentiable everywhere.

\begin{algorithm}[h]
\caption{: Lip\_conv$(K, N_\text{iter})$}
\label{algo:gram_iteration_conv}
\begin{algorithmic}[1]
  \STATE \textbf{Inputs} Filter: $K$, number of iterations: $N_\text{iter}$
  \STATE $D \gets \texttt{FFT2}(K) $ \hfill \text{// FFT}
  \STATE $r \gets 0_{n^2}$ \hfill \text{// initialize rescaling}
  \STATE \textbf{for} $1 \ldots N_\text{iter}$
  \STATE \quad \textbf{for} \text{i} \textbf{in} $1 \ldots n^2$ \hfill \text{// for-loop in parallel}
    \STATE  \quad \quad $r_i \gets 2
    (r_i + \log\norm{D_i}_F)$ \hfill \text{// cumulate rescaling}
    \STATE \quad \quad $D_i \gets D_i/\norm{D_i}_F ~ $ \hfill \text{// rescale to avoid overflow}
    \STATE \quad \quad $D_i \gets D_i^* D_i$ \hfill \text{// Gram iteration}
  \STATE  $\sigma_1 \gets \max_i \{ \norm{D_i}_F^{2^{-N_\text{iter}}} \exp{(2^{-N_\text{iter}} r_i}) \}$
  \STATE \textbf{return} $\sigma_1$
\end{algorithmic}
\end{algorithm}

Using Proposition~\ref{prop:upper_bound_conv} we devise Algorithm~\ref{algo:gram_iteration_conv} to compute spectral norm for a convolutional layer defined by a convolutional filter $K$ padded towards input size.
This algorithm is deterministic and precise due to the properties of Gram iteration, as well as tractable and scalable for convolutions as we consider a batch of $n^2$ relatively small matrices of size $\cout \times \cin$. The batch computation on GPU is well optimized hence the speed of the algorithm.

It is possible to consider smaller $n_0 \leq n$ in Equation~(\ref{equation:bound_conv_gram}) to further reduce computation for large input spatial size and still compute a bound as explained in Appendix~\ref{section:approx_n0}. Approximation between circular and zero paddings is discussed in Appendix~\ref{section:approx_circ}.

\section{Numerical experiments}
\label{section:experiments}

For research reproducibility,
the code is available
\url{https://github.com/blaisedelattre/lip4conv}.
All experiences were done on one NVIDIA RTX A6000 GPU.


\subsection{Computation of spectral norms}
\label{subsection:expe_comput_spectral_norm}

\begin{figure}[t]
    \centering
\includegraphics[scale=0.55]{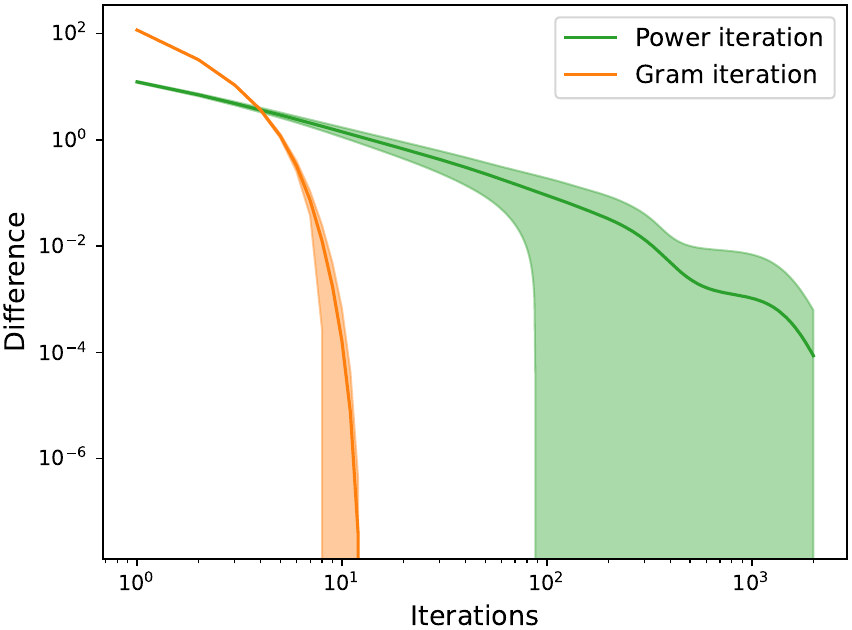}
    \caption{Convergence plot in log-log scale for spectral norm computation, comparing Power iteration and Gram iteration, one standard deviation shell is represented in a light color.
    Difference at each iteration is defined as $| {\sigma_1}_{\text{method}} - {\sigma_1}_{\text{ref}} |$.
    Gram iteration converges to numerical 0 in less than 12 iterations while Power iteration has not yet converged with 2,000  iterations and has a larger dispersion through runs.}
\label{fig:convergence_plot_gram_vs_power_itreration}
\end{figure}

We compare PI and GI methods for spectral norms computation, and a set of $100$ random Gaussian matrices $2000 \times 1000$ is generated.
The reference value for the spectral norm of a matrix ${\sigma_1}_{\text{ref}}$ is obtained from PyTorch using SVD, and Algorithm~\ref{algo:gram_iteration_dense} is used for GI. 
Estimation error is defined as $| {\sigma_1}_{\text{method}} - {\sigma_1}_{\text{ref}} |$.
We observe in Figure~\ref{fig:convergence_plot_gram_vs_power_itreration} that PI needs more than 2,000 iterations to fully converge to numerical 0 and at minimum $90$ iterations. Runs have a much larger standard deviation in comparison to GI. For that experiment, full convergence required up to 5,000 for some realizations, see Figure~\ref{fig:convergence_plot_gram_vs_power_itreration_5000_iters} in Appendix. 
This highlights how PI convergence can be very slow from one run to another for different generated matrices. 
Furthermore, in Figure~\ref{fig:convergence_plot_gram_vs_power_itreration_same_mat} in Appendix we show convergence of methods when a generated matrix is fixed: variance PI observed for PI is only due to its non-deterministic behavior, unlike GI which is deterministic.
This shows how the variance of PI is sensitive due to two causes: random generation of input matrices and the method in itself with random vector initialization at the start of the algorithm. On the contrary, GI converges under 15 iterations in every case. 

\begin{figure}[t]
    \centering
    \includegraphics[scale=0.55]{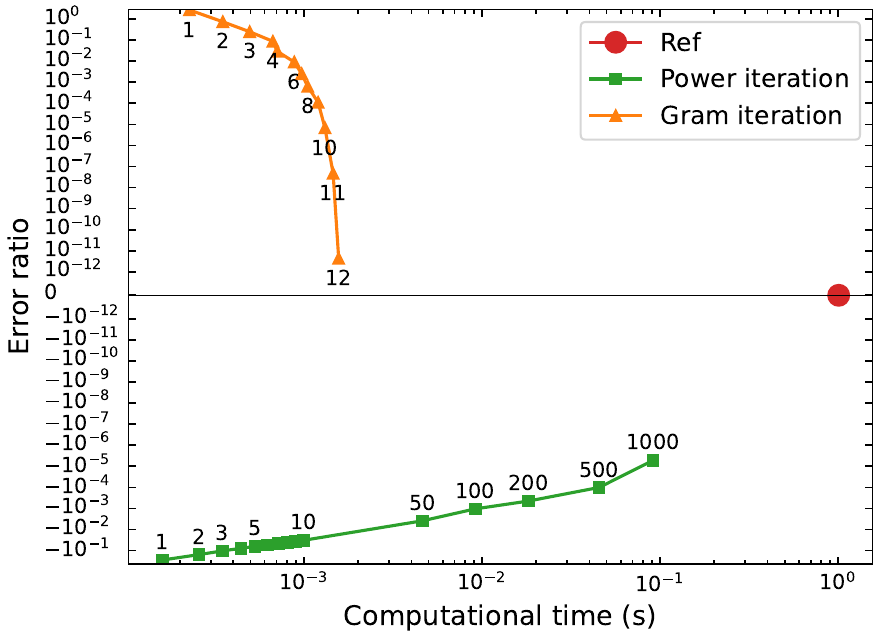}
    \caption{Error ratios over computational times of methods for spectral norm computation. Error ratio is defined as ${\sigma_1}_{\text{method}} / {\sigma_1}_{\text{ref}} - 1$. Points are annotated with the number of iterations. GI converges in $\pm 10^{-3}$s to numerical 0, while PI convergence is slower.}
\label{fig:comparaison_methods_compute_spectral_norm_matrix}
\end{figure}

To capture the sign of the error, \ie if the current estimator is an upper bound or not, we define the error ratio as ${\sigma_1}_{\text{method}} / {\sigma_1}_{\text{ref}} - 1$, depicted in
Figure~\ref{fig:comparaison_methods_compute_spectral_norm_matrix} also reported in Table~\ref{tab:comparaison_methods_compute_spectral_norm_matrix}.
We observe that the error ratio for PI is negative at each iteration, meaning that PI estimation approaches the reference while being inferior to it, thus being a lower bound.
GI empirically enjoys a very fast convergence, with $10$ iterations: doing one more iteration divide the error ratio by approximately $1e2$ and two more iterations by $1e4$. In contrast, PI has a much slower convergence: passing from $10$ to $100$ iterations only divide the ratio by a factor of $10$. 
This faster convergence rate from GI makes the method much faster in terms of computational time than PI and SVD for the same precision.
Those results motivate the use of GI for dense layers as well.


\begin{figure*}[!ht]
     \centering
     \begin{subfigure}[b]{0.32\textwidth}
         \centering
         \includegraphics[scale=0.4]{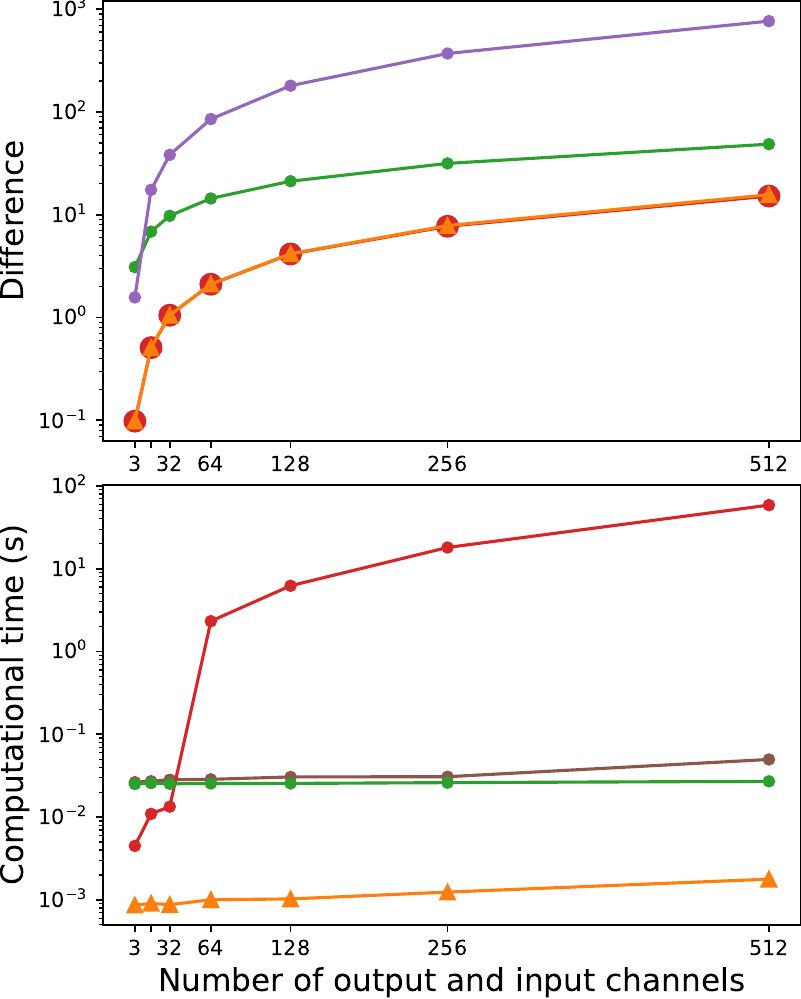}
         \caption{Varying number of channels $\cin$ and $\cout$, with $n=32$ and $k=3$.}
         \label{fig:times_estimation_error_vary_cin_cout}
     \end{subfigure}
     \hfill
     \begin{subfigure}[b]{0.32\textwidth}
         \centering
         \includegraphics[scale=0.4]{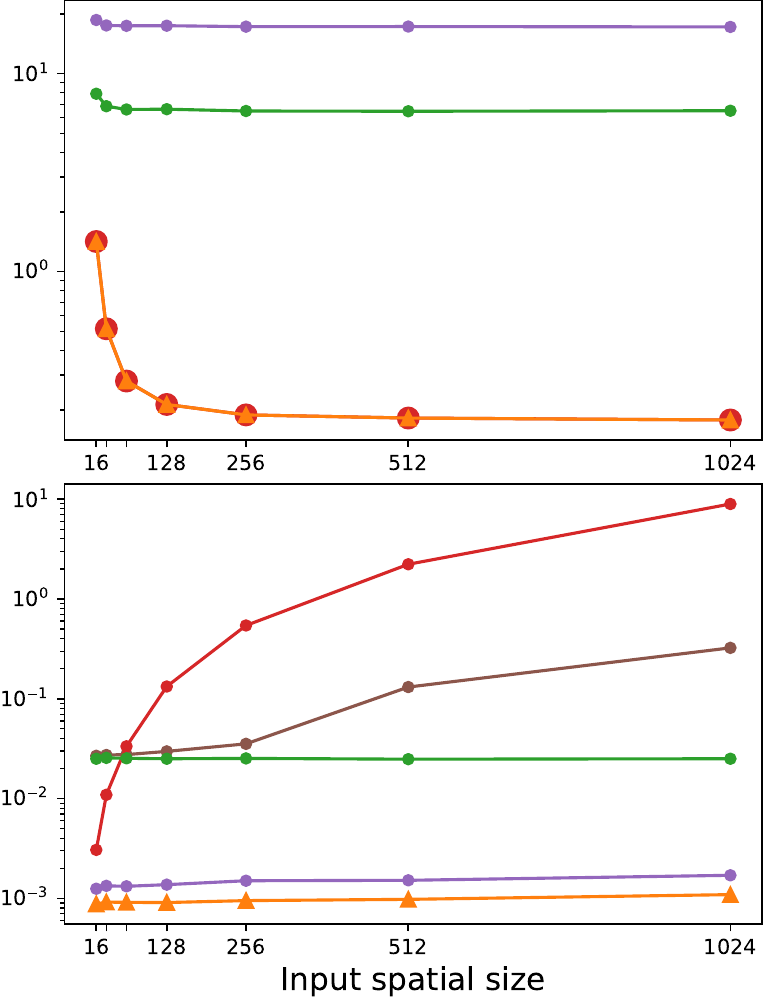}
         \caption{Varying input size $n$, with $\cin=16$, $\cin=16$ and $k=3$.}
         \label{fig:times_estimation_error_vary_input_size}
     \end{subfigure}
     \hfill
     \begin{subfigure}[b]{0.32\textwidth}
         \centering
         \includegraphics[scale=0.4]{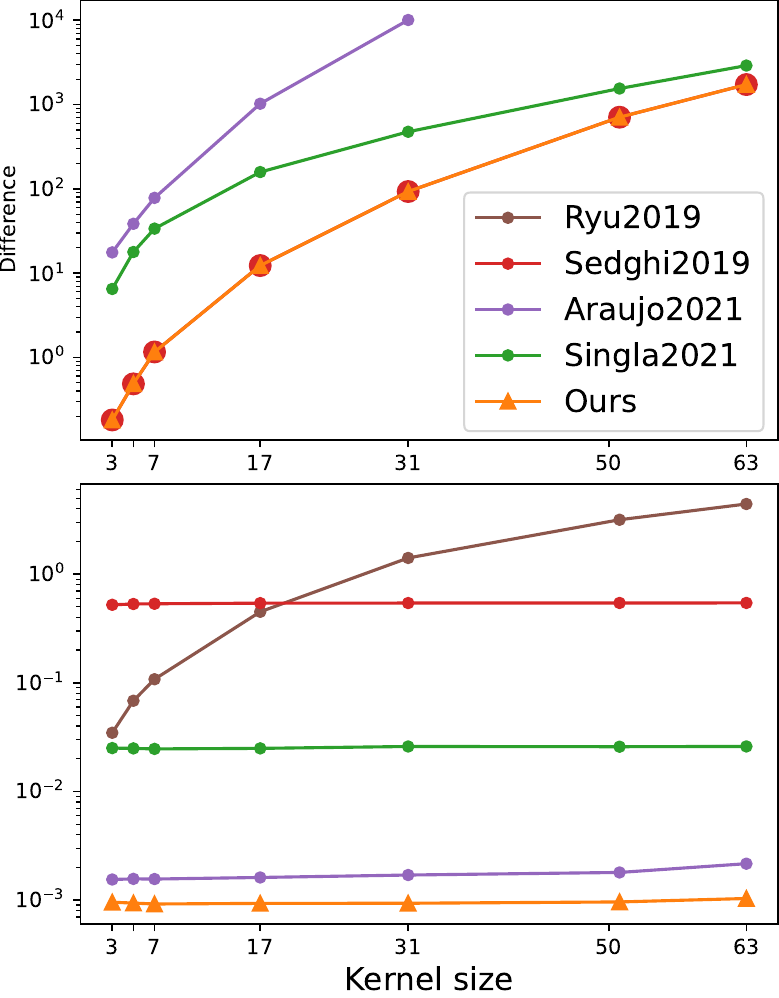}
         \caption{Varying kernel size $k$, with $\cin=16$, $\cin=16$ and $n=252$.}
         \label{fig:times_estimation_error_vary_kernel_size}
     \end{subfigure}
        \caption{For each subplot, we vary a convolution layer parameter and display the difference and computational times of Lipschitz bounds. Difference is expressed as  ${\sigma_1}_{\text{method}} - {\sigma_1}_{\text{Ryu2019}}$. \citet{sedghi2019singular} time cost overall is prohibitive for its use, regarding \citet{Ryu19Plug} it scales strongly with input and kernel size. 
        \citet{araujo2021lipschitz} and \citet{singla2021fantastic} are intermediate trade offs between precision and speed. 
        Our method performs best in terms of speed and  precision regarding reference \citet{Ryu19Plug} and gives the same values as \citet{sedghi2019singular} ones.
        }
        \label{fig:times_estimation_error_three_figures}
\end{figure*}

\subsection{Estimation on simulated convolutional layers}
\label{subsection:expe_estim_sim_conv_cout}

This experiment assesses the properties of bounds from different methods in terms of precision and computational time. We consider convolutional layers with constant padding and make vary three parameters: number of input and output channels (Figure~\ref{fig:times_estimation_error_vary_cin_cout}), spatial input size $n$ (Figure~\ref{fig:times_estimation_error_vary_input_size}) and  kernel size $k$ (Figure~\ref{fig:times_estimation_error_vary_kernel_size}).
For each experiment, random kernels are generated from Gaussian, and uniform distribution is repeated 50 times. We take average of errors, defined ${\sigma_1}_{\text{method}} - {\sigma_1}_{\text{Ryu2019}}$, and computational time (in seconds).
Method of \citet{Ryu19Plug} is taken as a reference with $100$ iterations to have precise estimation, \cite{singla2021fantastic} with $50$ iterations as in their code base, \cite{araujo2021lipschitz} with $50$ samples ($10$ taken in paper experiments description), and ours is Algorithm~\ref{algo:gram_iteration_conv} with $5$ iterations. We observe that our method gives a Lipschitz bound which matches the Lipschitz constant under the circular convolution hypothesis given by \cite{sedghi2019singular}, which gives the exact spectral norm for circular padding convolution. Furthermore, our method is the fastest of the considered methods.
We see that \cite{Ryu19Plug} computational time scales very fast with $n$ and $k$  whereas other methods having circular hypothesis remain quasi constant regarding computational times. We note that the Lipschitz difference increases as kernel size varies and the difference between circular and zero padding is more pronounced.
These experiments highlight the scalability of our method with respect to the parameters of convolutional layers.


\subsection{Estimation on convolutional layers of CNNs}

\begin{figure}[t]
    \centering
    \includegraphics[scale=0.6]{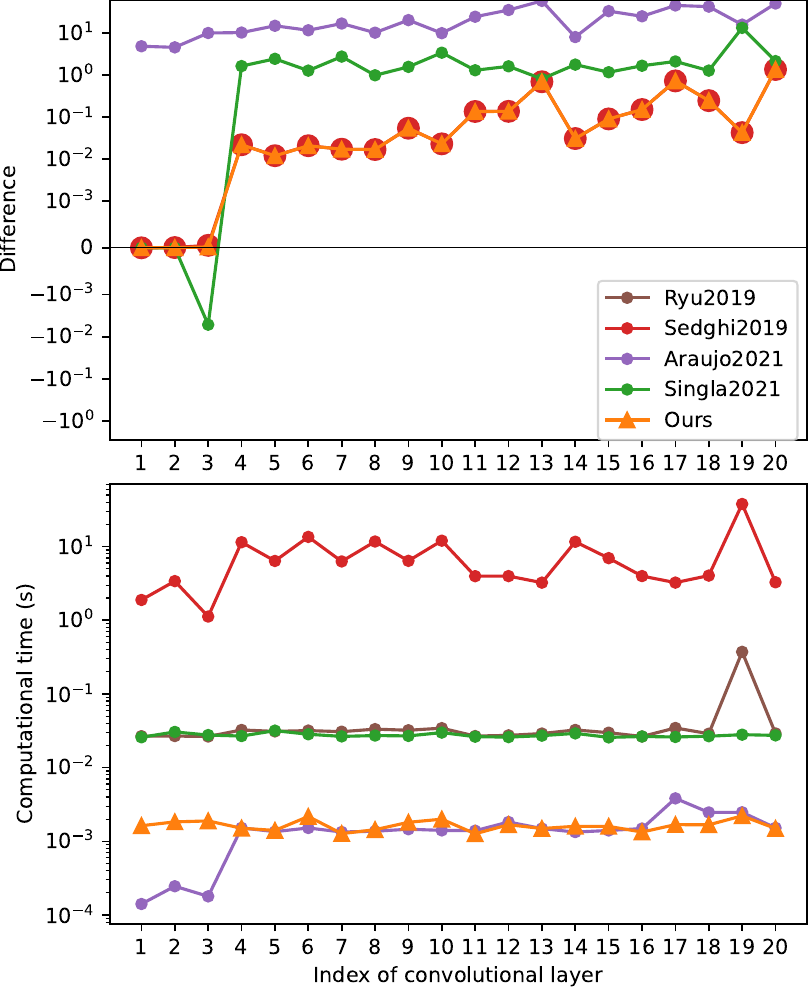}
    \caption{Comparison of Lipschitz bounds for all convolutional layers pre-trained ResNet18, reference is given by \citet{Ryu19Plug} method. Each convolutional layer in ResNet18 has a different number of output and input channels, kernel, and input spatial size. This experiment reproduces the usage of bounds on real kernel filters, we observe that our method gives the best precision of all and comparable time with \citet{araujo2021lipschitz}.}
    \label{fig:resnet18_convlip}
\end{figure}

Inspired by \cite{singla2021fantastic}, this experiment estimates the Lipschitz constant for each convolutional layer of a ResNet18 \cite{he2016deep}, pre-trained on the ImageNet-1k dataset. We take the same method's parameters as in Section~\ref{subsection:expe_estim_sim_conv_cout}.
Figure~\ref{fig:resnet18_convlip} reports the difference between reference and Lipschitz bounds for all convolutional layers of ResNet18 for different methods.
We observe that our method gives a bound very close to \cite{sedghi2019singular} and always above the reference value given by \cite{Ryu19Plug}. 
For the filter of index $3$, we see that the value given by \cite{singla2021fantastic}, $2.03$, is below the reference value $2.05$: this illustrated the issue with using PI in the pipeline to obtain an upper bound. Numerical results are detailed in Table~\ref{tab:lipschitz_resnet18}.

\begin{table*}[t]
\caption{Comparison of networks bound Lipschitz ratio with standard deviation for several CNNs, reference for computing the ratio is the bound given by \citet{Ryu19Plug} method. Overall Lipschitz constant bound $B_{\text{method}}$ is estimated for each method. Ratio of network Lipschitz bound is defined as $B_{\text{method}} / B_{\text{Ryu2019}}$. Results are averaged over 100 runs.
Ratio standard deviations of ours, \citet{araujo2021lipschitz} and \citet{sedghi2019singular} are induced by \citet{Ryu19Plug} method. We give the same ratio as \citet{sedghi2019singular} in a much lower time.
}
\vspace{0.1cm}
\centering
\sisetup{%
    table-align-uncertainty=true,
    separate-uncertainty=true,
    detect-weight=true,
    detect-inline-weight=math
  }
\vspace{0.1cm}
\resizebox{2.1\columnwidth}{!}{%
\begin{tabular}
    {
    l
    l
    l
    l
    l
    l
    }
\toprule
\multirow{2}[2]{*}{\textbf{Model}} &\multicolumn{4}{c}{\textbf{Ratio of Network Lipschitz Bound  (Total Running Time (s))}} \\
\cmidrule{2-6}
& \textbf{Ours} & \citeauthor{singla2021fantastic} &\citeauthor{araujo2021lipschitz} &\citeauthor{sedghi2019singular} & \citeauthor{Ryu19Plug}\\
\midrule
VGG16     & $1.14 \pm 0.020 ~(0.100)$ & $23.90 \pm 6.80 ~(0.47)$  & $4.88\mathrm{e}{+09}\pm 4.83\mathrm{e}{+09} ~(0.027) $ & $1.14 \pm 0.020 ~(525)$ & $(0.64)$ \\
VGG19     & $1.16 \pm 0.005 ~(0.110)$ & $30.63 \pm 0.30 ~(0.53)$  & $9.70\mathrm{e}{+09}\pm 4.84\mathrm{e}{+07} ~(0.030)$  & $1.16 \pm 0.005 ~(639)$ & $(0.71)$ \\
ResNet18  & $1.47 \pm 0.007 ~(0.039)$ & $87.93 \pm 0.88 ~(0.50)$  & $3.03\mathrm{e}{+10}\pm 1.56\mathrm{e}{+08} ~(0.039)$  & $1.47 \pm 0.007 ~(185)$ & $(0.71)$ \\
ResNet34  & $1.82 \pm 0.350 ~(0.060)$ & $4982 \pm 4894 ~(0.88)$   & $7.86\mathrm{e}{+21}\pm 7.86\mathrm{e}{+21} ~(0.050)$ & $2.15 \pm 0.350 ~(237)$ & $(1.16)$ \\
ResNet50  & $1.68 \pm 0.350 ~(0.100)$ & $3338 \pm 4622 ~(1.05)$   & $6.66\mathrm{e}{+31}\pm 9.42\mathrm{e}{+31} ~(0.050)$ & $1.67 \pm 0.350 ~(377)$ & $(1.49)$ \\
ResNet101 & $1.74 \pm 0.320 ~(0.173)$ & $4026 \pm 4178 ~(1.50)$   & $6.55\mathrm{e}{+64}\pm 1.13\mathrm{e}{+65} ~(0.100)$ & $1.74\pm 0.320 ~(551)$ & $(2.20)$ \\
ResNet152 & $1.92 \pm 0.460 ~(0.260)$ & $8.39\mathrm{e}{+4} \pm 1.60\mathrm{e}{+5} ~(2.05)$  &$2.85\mathrm{e}{+96} \pm 5.71\mathrm{e}{+96} ~(0.120)$ &$1.92 \pm 0.460 ~(725)$ & $(3.01)$ \\
\bottomrule
\end{tabular}
}
\label{tab:lipschitz_product_ratio_ryu}
\end{table*}

Contrary to the first part studying only convolutional layers, the overall Lipschitz constant of the network is assessed here. We consider several CNNs pre-trained on the ImageNet-1k dataset: VGG \cite{simonyan2014vgg} and ResNet \cite{he2016deep} of different sizes.
Using rules from Appendix~\ref{section:lip_cnn} to compute the Lipschitz constant of each individual layer in CNN, we compute a bound $B$ on the overall Lipschitz constant of the network.
Then we compare this produced bound for each method  to the one obtained by \cite{Ryu19Plug} by considering the ratio: $B_{\text{method}} / B_{\text{Ryu2019}}$.
We take $100$ iterations for \cite{Ryu19Plug} to have a precise reference, $50$ iterations for \cite{singla2021fantastic}, $20$ samples for \cite{araujo2021lipschitz} and $7$ iterations for our method, as the task requires increased precision. Results are reported in Table~\ref{tab:lipschitz_product_ratio_ryu}, we see that our method gives results similar to \cite{sedghi2019singular} and overall  network Lipschitz bound is tighter  than \cite{singla2021fantastic} and \cite{araujo2021lipschitz} in comparison to \cite{Ryu19Plug}. This experience illustrates that small errors in the estimation of a single Lipschitz constant layer can lead to major errors on the Lipschitz bound of the overall network and that acute precision is crucial. Moreover, methods using PI such as \cite{singla2021fantastic} have a huge standard deviation in this task which is problematic for deep networks whereas ours and \cite{sedghi2019singular} have standard deviations that remain small. Our method offers the same precision quality as \cite{sedghi2019singular} but significantly faster.


\subsection{Regularization of convolutional layers of ResNet}
\label{subsection:reg_conv_layers}

Regularization of spectral norms of convolutional layers has been studied in previous works \cite{yoshida2017spectral, miyato2018spectral, gouk2021regularisation}. The goal of this experiment is to assess how different bounds control the Lipschitz constant $L_j$ of each convolution layer $\phi_j$ along the training process. To measure this control, a target Lipschitz constant $L$ is set for all convolutions. The loss optimized during training becomes:
$\mathcal{L_{\text{train}}} + \mu_{\text{reg}} \sum_{j=1}^l \mathcal{L_{\text{reg}}}(\phi_j)$, with 
\begin{equation}
\mathcal{L_{\text{reg}}}(\phi_j) = {\sigma_1}_\text{method}(\phi_j) ~ \mathbbm{1}_{{\sigma_1}_\text{method}(\phi_j) > L} \ .
\end{equation}
This regularization loss penalizes the Lipschitz constant of the convolutional layer only if it is above the target Lipschitz constant $L$, else regularization is deactivated. Function $x \mapsto \mathbbm{1}_{x > L}$ is not differentiable in $L$, but the probability for $\sigma_1(\phi_j)$ to be numerically equal to $L$ is close to 0.

We use ResNet18 architecture \cite{he2016deep}, trained on the CIFAR-10 dataset for 200 epochs, and with a batch size of 256. We use SGD with a momentum of $0.9$ and an initial learning rate of $0.1$ with a cosine annealing schedule. 
The baseline is trained without regularizations.
Different regularizations are compared: Lipschitz regularization by different methods with $\mu_{\text{reg}} = 1e\text{-}1$; and the usual weight decay (WD) applied on convolutional filters with $\mu_{\text{reg}} = 5e\text{-}3$, which has been picked to bound layers spectral norm around the target as good as possible.
We only compare methods under the circular padding hypothesis and use \cite{sedghi2019singular} as the reference since it provides an exact value. We take $6$ iterations for Gram iteration in our method, $10$ for power iteration in \cite{singla2021fantastic}, and $10$ samples in \cite{araujo2021lipschitz} to have similar training times. For our method, we implement an explicit backward differentiation using Equation~(\ref{eq:gram_iteration_bound_gradient}) to speed up the gradient computation and reduce the memory footprint.

Figure~\ref{fig:histograms_convolution_resnet_training_for_bounds} presents results on Lipschitz constant control of convolutional layers for ResNet18. The goal here is to have a Dirac distribution at target value $L=1$, \ie we want each convolutional layer to have a spectral norm of $1$. We observe that the more precise a method is, the more well-controlled the resulting spectrum resulting from training is. Our regularization method matches each layer's target spectral norm of $1$.
Weight decay gives gross estimations and \cite{araujo2021lipschitz} overestimates spectral norms and thus results in over-constrained Lipschitz constant for some layers. \cite{singla2021fantastic} produces a maximum spectral of $1.1$ norm which is above target $L$ but overall Lipschitz constant is centered around the target with dispersion. This suggests that loosely bound under-constraints convolution's spectral norm and that the differentiation of our bound behaves correctly.
Detailed spectral norm histograms over epochs for each Lipschitz regularization method are presented in Figure~\ref{fig:spectral_norm_over_epochs_reg_resnet_training_for_bounds}, showing that all spectral norms are below the target Lipschitz constant at epoch $120$ for our method and accounts for fast convergence in Lipschitz regularization in comparison to other methods.


\begin{figure}[h]
    \centering
    \includegraphics[scale=0.43]{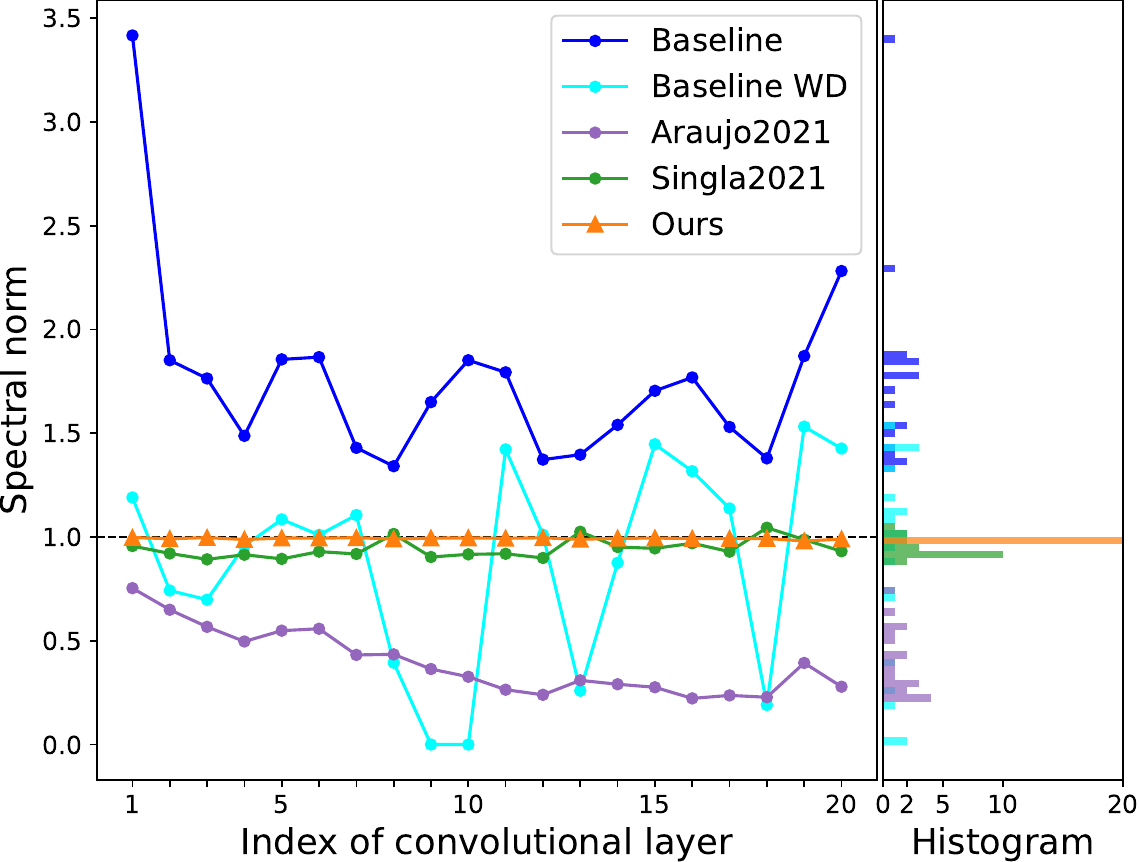}
    \caption{Plot and histogram of spectral norms for each convolution layer in ResNet18 at the end of training on CIFAR-10, for different regularization methods. We observe that the histogram of spectral norms regularized with our method is all at the target Lipschitz constant, meanwhile, other methods' histograms are scattered. }
    \label{fig:histograms_convolution_resnet_training_for_bounds}
\end{figure}


\subsection{Image classification with regularized ResNet}

Finally, to study the impact of regularization on image classification accuracy and training time, we conduct experiments using the same experimental training framework as previous Section~\ref{subsection:reg_conv_layers} for CIFAR-10. That is to say with a high learning rate and regularization parameter to discriminate the different bounds. Trainings are repeated 10 times. Accuracy of \cite{sedghi2019singular} was omitted because the training was too long to finish, around $6750$s per epoch. We vary the number of iterations for our bound, from $3$ to $6$ at the end of the training, because we observe reduced standard deviation by doing so. Our training time is reported for $6$ Gram iterations.
Results of experiments are reported in Table~\ref{tab:test_accuracies_method_reg_cifar10}. Most of the bounds for Lipschitz regularization \cite{Ryu19Plug, araujo2021lipschitz, singla2021fantastic} tend to slightly degrade accuracy in comparison to baseline, contrary to our bound where a small accuracy gain can be observed: our Lipschitz regularization is not a trade-off on accuracy under this training setting. Furthermore, the low standard deviation of our accuracies suggests that our regularization stabilizes training contrary to all other approaches. The computational cost of our method is the lightest of all Lipschitz regularization approaches.

\begin{table}[h]
\caption{This table shows test accuracies and training times for ResNet18 on CIFAR-10, repeated 10 times. Our method has a much narrow standard deviation, slightly better accuracy, and lower training time in comparison to other bounds for Lipschitz regularization.}
\centering
\footnotesize
\begin{tabular}{lllllll}\toprule
 &\textbf{Test accuracy} (\%) &  \textbf{Time per epoch (s)}  \\\cmidrule{1-3}
{Baseline} &93.32 $\pm$ 0.12 &11.4 \\
{Baseline WD} &93.30 $\pm$ 0.19 &11.4 \\
{\citeauthor{Ryu19Plug}} & {93.27} $\pm$ 0.13 & 95.0 \\
{\citeauthor{sedghi2019singular}} & {\ \ \ \ \ \ \ \ \ \ \xmark} & 6750 \\
{\citeauthor{araujo2021lipschitz}} &90.66 $\pm$ 0.26 &37.4 \\
{\citeauthor{singla2021fantastic}} &93.29 $\pm$ 0.15 &38.4 
\\\midrule
Ours &93.48 $\pm$ 0.08 & 31.9 \\\midrule
\end{tabular}
\label{tab:test_accuracies_method_reg_cifar10}
\end{table}

To demonstrate the scalability of our method, a ResNet18 is trained on the ImageNet-1k dataset, processing $256 \times 256$ images.
ResNet18 baseline (with WD) is compared with two regularized ResNet18: one trained from scratch on 88 epochs, and one initialized on a pre-trained baseline and fined-tuned on 10 epochs. Trainings are repeated four times on 4 GPU V100.
The results of experiments are reported in 
Table~\ref{tab:test_accuracies_method_reg_imagenet}.
This experiment shows that our bound remains efficient when dealing with large images. Moreover, it is not required to train a new ResNet from scratch: any trained ResNet can be fine-tuned on a few epochs to be regularized.

\begin{table}[h]
\caption{This table shows test accuracies and training times for ResNet18 on ImageNet-1k, repeated 4 times.}
\centering
\footnotesize
\begin{tabular}{lllllll}\toprule
 & \textbf{Test accuracy} (\%) & \textbf{Time per epoch (s)} \\\cmidrule{1-3}
{Baseline WD} &69.76 & 746 \\\midrule
Ours &70.77 $\pm$ 0.12 & 782 \\
Ours (fine-tuned) &70.50 $\pm$ 0.05 & 782 \\\midrule
\end{tabular}
\label{tab:test_accuracies_method_reg_imagenet}
\end{table}







%
%
%
%

\section{Conclusion}




In this paper, we introduce an efficient upper bound on the spectral norm of circular convolutional layers. It uses Gram iteration, a new and promising alternative to the power method, exhibiting quasi-quadratic convergence.
Our comprehensive set of experiments shows that our method is accurate (with low variance), fast and tractable for large input sizes and convolutional filters. Our result provides upper bound guarantees on the spectral norm. Furthermore, it allows the Lipschitz constant of convolutional layers to be precisely controlled by differentiation. It is also suitable for Lipschitz regularization during the training process of deep neural networks, bringing stability along with better classification accuracy.

Instead of focusing on the largest singular values, some works orthogonalize the weights \cite{li2019preventing, trockman2021orthogonalizing}, constraining all singular values to be equal to 1. We believe that our approach provides better expressiveness than orthogonalizing layers, as the spectrum can retain more information. Further work will consist of a theoretical study of the differences between the two approaches.

\subsubsection*{Acknowledgments}
This work was performed using HPC resources from GENCI- IDRIS (Grant 2023-AD011014214) and funded by the French National Research Agency (ANR SPEED-20-CE23-0025).

\newpage
\bibliography{bibliography}
\bibliographystyle{icml2023}

\clearpage
\appendix
\onecolumn


\section{Iterative methods to estimate spectral norm}
\label{section:spectral_norm_methods}

Input $G$ is a rectangular complex matrix, $G \in \mathbb{C}^{p \times q}$.
Concerning the termination criterion, $N_\text{iter}$ is the maximal number of iterations, but it could be combined with a threshold on the update norm.


\subsection{Power iteration}

Power iteration is described in Algo.~\ref{algo:power_iteration}: it is a simple but non-deterministic method that may converge slowly depending on input $G$ and on initialization of random vector $u \in \mathbb{C}^{q}$.




\subsection{Gram iteration}

A naive implementation of Gram iteration is given in Algo.~\ref{algo:gram_iteration}: this simple method is deterministic.
If each Gram iteration is more complex than the power one, the matrix-matrix product is highly optimized on GPU, providing a fast method.
If we compare PI and GI, we see that in the latter we iterate on the operator represented by matrix $G$ and modify it, whereas in PI operator remains the same.

\begin{algorithm}[h]
\caption{: Gram\_iteration\_naive$(G, N_\text{iter})$}
\label{algo:gram_iteration}
\begin{algorithmic}[1]
  \STATE \textbf{Inputs} matrix: $G$, number of iterations: $N_\text{iter}$
  \STATE \textbf{for} $1 \ldots N_\text{iter}$
    \STATE \quad $G \gets G^* G$
  \STATE $\sigma_1 \gets \norm{G}_F^{2^{-N_\text{iter}}}$
  \STATE \textbf{return} $\sigma_1$
\end{algorithmic}
\end{algorithm}


\subsection{Gram iteration with Rescaling}
\label{section:gram_iteration_rescaling}

Due to the fast increase of values, we observe overflow in Algo.~\ref{algo:gram_iteration}. To avoid it, we introduce a rescaling by Frobenius norm at each iteration: $G \gets G/\norm{G}_F$. 

Denoting $G^{(t)}$ the Gram iterate at iteration $t$, the result of Gram iteration with rescaling is:
\begin{equation*}
    \norm{G^{(N_{iter})}}_F^{2^{-N_\text{iter}}} 
    = \norm{(G^* G)^{{2^{N_\text{iter}-1}}}}_F^{2^{-N_\text{iter}}} \ / \ \prod_{t=1}^{N_{iter}} \norm{G^{(t)}}_F^{2^{-t+1}}
    = s_{2^{N_\text{iter}}}(G) \ / \ \prod_{t=1}^{N_{iter}} \norm{G^{(t)}}_F^{2^{-t+1}} \ ,
\end{equation*}
which is no longer equal to the expected $2^{N_\text{iter}}$-Schatten norm of $G$.

In order to unscale the result at the end of the method, scaling factors are cumulated into a variable $r$. Denoting $r^{(t)}$ the cumulated scaling factor at iteration $t$, we have for the last iteration: 
\begin{equation*}
 r^{(N_{iter})} = \sum_{t=1}^{N_{iter}} 2^{N_{iter}-t+1} \log \norm{G^{(t)}}_F \ .
\end{equation*}
At the end of the method, we must unscale by: 
\begin{equation*}
 \exp{ \left(2^{-N_\text{iter}} ~r^{(N_{iter})}\right)} = \prod_{t=1}^{N_{iter}} \norm{G^{(t)}}_F^{2^{-t+1}} \ .
\end{equation*}
Finally, the result of Gram iteration with rescaling at each iteration and final unscaling is the right $2^{N_\text{iter}}$-Schatten norm of $G$:
\begin{equation*}
    \norm{G^{(N_{iter})}}_F^{2^{-N_\text{iter}}} \exp{\left(2^{-N_\text{iter}} ~r^{(N_{iter})}\right)} 
    = \norm{(G^* G)^{2^{N_\text{iter}-1}}}_F^{2^{-N_\text{iter}}} 
    = s_{2^{N_\text{iter}}}(G) \ .
\end{equation*}
Unscaling is crucial as it is required to obtain a Schatten norm, which remains a strict upper bound on the spectral norm at each iteration of the method.

Gram iteration with rescaling is described in Algo.~\ref{algo:gram_iteration_dense} and can be applied on a dense layer as it is.




\subsection{Largest singular vectors by Gram iteration}

At the end of the Gram iteration, the largest singular vectors can be estimated.
The largest right-singular vector $u \in \mathbb{C}^{q}$ can be estimated using the final matrix $G \in \mathbb{C}^{q \times q}$, as it is a projector on the space of maximal singular value. So any non-zero column of index $i$, $G_{:, i}$ will be multiples of $u$: 
\begin{equation}
u \leftarrow G_{:, i} / \norm{G_{:, i}}_2 \ .
\end{equation}

Then, defining $G_0$ the initial matrix $G$, the largest left-singular vector $v \in \mathbb{C}^{p}$ is:
\begin{equation}
v \leftarrow G_0 u / \norm{G_0 u}_2 \ .
\end{equation}

Note that using deflation techniques with GI, it is possible to estimate all other singular values and associated singular vectors, as done classically with PI.


\subsection{Largest eigenpair by Gram iteration}

Now, input $G$ is a square complex matrix, $G \in \mathbb{C}^{p \times p}$.
We want to compute its largest eigenpair $(\lambda_1, u_1)$,
where $\lambda_1$ is the largest eigenvalue and $u_1$ the largest eigenvector.

Power iteration described in Algorithm~\ref{algo:power_iteration} can be simply adapted to output this largest eigenpair.
Steps 4 and 5 are replaced by: $u \gets Gu / \norm{Gu}_2$, and step 6 is replaced by: $\lambda_1 \gets u^*Gu / \norm{u}_2$. At the end of PI, $u$ is $u_1$.

Similarly, it is also possible to adapt the Gram iteration described in Algorithm~\ref{algo:gram_iteration_dense}.
Step 5 is replaced by: $G \gets GG$. 
Given the eigenvalue decomposition of $G = Q \Lambda Q^{-1}$, the $t^{th}$ iterate $G^{(t)} = G^{2^{t-1}} = Q \Lambda^{2^{t-1}} Q^{-1}$. The same results applied for this alternative iteration as for Gram iteration, indeed proof of Section~\ref{sec:proofs} still holds.

Since Gram matrix $G^{(2)} = {G^{(1)}}^{*}G^{(1)}$ is Hermitian, \ie ${G^{(2)}}^{*}=G^{(2)}$, both steps 5 are equivalent from the second iteration: ${G^{(t)}}^{*}G^{(t)} = G^{(t)}G^{(t)}$, $\forall t \geq 2$.



\section{Proofs of Section~\ref{ssec:spectral_norm_conv}}
\label{sec:proofs}

Singular value decomposition of $G \in \mathbb{C}^{p \times q}$ is written $V_1 \Sigma V_2^*$, where $\Sigma = \diag(\sigma)$, $\sigma = (\sigma_i)_{1 \leq i \leq m}$ and $m = \min(p, q)$. We use the decreasing order indexation convention: $\sigma_1 \geq \sigma_2 \geq \dots \geq \sigma_m$.
\\
We have: $\sigma_1(G) = \sigma_1 = \norm{\sigma}_\infty$.
\\
We get: $G^{(1)}= G$, and $\forall t \geq 2$: $G^{(t)} = V_2 \Sigma^{2^{t-1}} V_2^*$.


\subsection{Proof of Theorem~\ref{thm:main_result}}
\label{proof:bound_sqrt_frob_gt}

\begin{prop*}
An upper bound of the spectral norm is:
\begin{equation*}
\sigma_1(G) \leq {\norm{G^{(t)}}_F}^{2^{1 -t}} \ \ .
\end{equation*}
\end{prop*}
\begin{proof}
We start using this property on the largest singular value: 
\begin{align*}
    \sigma_1(G)^2 &= \sigma_1(G^* G) \\
    \sigma_1(G) &= {\sigma_1(G^* G)}^{\frac{1}{2}} \ . \\
\text{Iterating $t$ times: \ \ }
    \sigma_1(G) &= {\sigma_1(G^{(t)})}^{2^{1 -t}}\\
    \sigma_1(G) &\leq {\norm{G^{(t)}}_F}^{2^{1-t}} \ .
\end{align*}
\end{proof}


\begin{prop*}
The bound sequence converges to the spectral norm:
\begin{equation*}
\norm{G^{(t)}}_F^{2^{1-t}} \underset{t \to +\infty}{\longrightarrow} \sigma_1(G) \ \ .
\end{equation*}
\end{prop*}

\begin{proof}
For $t \geq 1$,
\begin{align*}
 \norm{G^{(t)}}_F^2 &= \text{Tr}({G^{(t)}}^* G^{(t)}) \\
 \norm{G^{(t)}}_F^2 &= \text{Tr}(G^{(t+1)})\\
 \norm{G^{(t)}}_F^2 &= \text{Tr}(V_2 |\Sigma|^{2^{t}} V_2^*)\\
 \norm{G^{(t)}}_F^2 &= \text{Tr}(V_2^* V_2 |\Sigma|^{2^{t}} )\\
 \norm{G^{(t)}}_F^2 &= \text{Tr}(|\Sigma|^{2^{t}}) \\
 \norm{G^{(t)}}_F^2 &= \sum_{i=1}^m |\sigma_i|^{2^{t}}
\end{align*}
\begin{align*}
    ({||G^{(t)}||_F^2})^{2^{-t}} &= \left({\sum_{i=1}^m |\sigma_i|^{2^{t}}}\right)^{2^{-t}}\\
    ||G^{(t)}||_F^{2^{1-t}} &= \norm{\sigma}_{2^{t}} \underset{t \to +\infty}{\longrightarrow} \norm{\sigma}_\infty\\
\end{align*}
and $\norm{\sigma}_\infty = \sigma_1(G)$ .
\end{proof}


\begin{defn}
Q-convergences.\\ 
    Suppose sequence $(u_t)_{t \in \mathbb{N}}$ converges to $\hat{u}$, it converges Q-linearly if
    $\underset{t \to +\infty}{\text{lim}} \frac{|u_{n+1} - \hat{u} |}{|u_{n} - \hat{u} |} = \mu < 1$. \\
    It converges Q-superlinearly if $\mu = 0$.
    The number $\mu$ is called the rate of convergence. \\
    It converges Q-quadratically if  $\underset{t \to +\infty}{\text{lim}} \frac{|u_{n+1} - \hat{u} |}{|u_{n} - \hat{u} |^2} = 0$ .
\end{defn}


\begin{prop*}
The convergence of bound sequence is Q-superlinear of order $2 - \epsilon$, for $\epsilon$  arbitrary small (quasi quadratic convergence in practice).
When $\sigma_1(G) = \sigma_2(G)$, convergence is Q-linear of rate $\frac{1}{2}$.
\end{prop*}

\begin{proof}
To analyze convergence we compute the quotient and look at its limit to conclude on convergence type.

\begin{align*}
    \norm{\sigma}_{2^t} - \norm{\sigma}_\infty &= \left(\sum_{i=1}^m |\sigma_i|^{2^{t}} \right)^{2^{-t}} - \norm{\sigma}_\infty \\
    \norm{\sigma}_{2^t} - \norm{\sigma}_\infty &= 
    \left(
        \left(\sum_{i=1}^m \left|\frac{\sigma_i}{\sigma_1} \right|^{2^{t}} \right)^{2^{-t}} 
        - 1
    \right)
    \norm{\sigma}_\infty \ .
\end{align*}

\begin{align*}
    \text{We define} \ \ \
    \epsilon_t &=\left(\sum_{i=1}^m \left|\frac{\sigma_i}{\sigma_1}\right|^{2^{t}} \right)^{2^{-t}} - 1
    \\
    \epsilon_t &= \exp{
        \left(
            2^{-t}
            \log\left(\sum_{i=1}^m \left|\frac{\sigma_i}{\sigma_1}\right|^{2^{t}} \right) 
        \right)}
        - 1 \\
         \epsilon_t &= \exp{
        \left(
            2^{-t}
            \log\left(1 + \sum_{i=2}^m \left|\frac{\sigma_i}{\sigma_1}\right|^{2^{t}} \right) 
        \right)}
        - 1 \ .
        \\
        \text{With} \
        \alpha = \sum_{i=2}^m \left|\frac{\sigma_i}{\sigma_1}\right|^{2^{t}} \underset{t\to +\infty}{\longrightarrow} 0, \ &\text{using series expansion:} \ \\
        \epsilon_t &\underset{t \to +\infty}{=} \exp{
        \left(
            2^{-t}
            \left( \alpha + o(\alpha) \right) 
        \right)}
        - 1 \\
        \epsilon_t &\underset{t \to +\infty}{=} 
            2^{-t} \alpha + o(2^{-t}\alpha) \ .
\end{align*}

For $q > 0$,
\begin{align*}
    \frac{\epsilon_{t+1}}{\epsilon_{t}^q} &\underset{t \to +\infty}{\sim} 
    \frac
    {2^{-(t+1)}
        \sum_{i=2}^m \left|\frac{\sigma_i}{\sigma_1}\right|^{2^{t+1}}
    }
    {2^{-tq} 
    \left(
        \sum_{i=2}^m \left|\frac{\sigma_i}{\sigma_1}\right|^{2^{t}}
    \right)^q}\\
    \frac{\epsilon_{t+1}}{\epsilon_{t}^q} &\underset{t \to +\infty}{\sim} 
    2^{t(q -1) - 1}
    \frac
    {
        \left|\frac{\sigma_2}{\sigma_1}\right|^{2^{t+1}}
    }
    { 
        \left|\frac{\sigma_2}{\sigma_1}\right|^{q2^{t}}
    }\\
    \frac{\epsilon_{t+1}}{\epsilon_{t}^q} &\underset{t \to +\infty}{\sim} 
    2^{t(q -1) - 1}
    \left|\frac{\sigma_2}{\sigma_1}\right|^{2^{t}(2 - q)} \ .
\end{align*}

When $\sigma_1 > \sigma_2$, for $0 < q < 2$,
$\underset{t \to +\infty}{\lim} \frac{\epsilon_{t+1}}{\epsilon_{t}^q} = 0$.
Hence convergence is Q-superlinear of order $q= 2-\epsilon$, where $\epsilon$ is arbitrarily small, providing a quasi-quadratic convergence in practice.

For the special case $\sigma_1 = \sigma_2$, convergence is Q-linear ($q=1$) of rate $\frac{1}{2}$.

\end{proof}


\subsection{Proof of Proposition \ref{prop:gradient_gram_iteration}}
\label{section:gradient_gram_iteration}

To implement explicit differentiation instead of auto-differentiation, we need to compute the gradient of the spectral bound given by Gram iteration:
$||G^{(t)}||_F^{2^{1-t}}$ .

Compute the gradient of $G \mapsto ||G^{(t)}||_F^{2^{1-t}}$ .

For $t=1$, $G^{(1)} = G$,
\begin{equation*}
\frac{\partial \norm{G}_F}{\partial G} = \frac{G}{\norm{G}_F} \ .
\end{equation*}

For $t \geq 2$, note that $G^{(t)} = (G^* G)^{2^{(t-2)}}$ .
First, we calculate:
\begin{equation*}
 \frac{\partial \norm{G^{(t)}}^2_F}{\partial G} = 2^t G ( G^* G)^{2^{t-1} -1} \ . 
\end{equation*}
Then,
\begin{align*}
 \frac{\partial (\norm{G^{(t)}}^2_F)^{1/2^t}}{\partial G} &= \frac{1}{2^t}(\norm{G^{(t)}}^2_F)^{1/2^t - 1} 2^t ~ G ( G^* G)^{2^{t-1} -1}   \\
 \frac{\partial (\norm{G^{(t)}}^2_F)^{1/2^t}}{\partial G} &=(\norm{G^{(t)}}_F)^{1/2^{t-1} - 2} ~ G ( G^* G)^{2^{t-1} -1}   \\
  \frac{\partial (\norm{G^{(t)}}^2_F)^{1/2^t}}{\partial G} &=\frac{G ( G^* G)^{2^{t-1} -1}}{(\norm{G^{(t)}}_F)^{2(1 - 2^{-t})}} \ .
\end{align*}

Finally using the chain rule, the gradient is given for $t \geq 1$:
\begin{equation}
    \label{eq:gradient_formulae_apendix}
     \frac{\partial (\norm{G^{(t)}}^2_F)^{1/2^t}}{\partial G} =\frac{G ( G^* G)^{2^{t-1} -1}}{(\norm{G^{(t)}}_F)^{2(1 - 2^{-t})}} \ .
\end{equation}


To avoid overflow in differentiation, one can rescale $G$ at the start of the gradient calculation. As the spectral norm $\norm{G}_2$ has been computed during the forward pass, it can be used for rescaling. Then, calculate the gradient with Equation~(\ref{eq:gradient_formulae_apendix}), and finally unscale the result by the factor: $1 / \norm{G}_2$.


\section{Approximations for convolutional layers}
\label{section:approx}

\subsection{Approximation for lower input spatial size $n_0 \leq n$}
\label{section:approx_n0}

A bound on spectral norm can be produced with Gram iteration as in Proposition~\ref{prop:upper_bound_conv}:
\begin{equation*}
  \sigma_1(W)
  = \max_{1 \leq i \leq n^2} \sigma_1(D_i) 
  \leq \max_{1 \leq i \leq n^2} \norm{{D}^{(t)}_i}_F^{2^{1-t}} \ .
\end{equation*}

To consider the very large input spatial size $n$, we can use our bound by approximating the spatial size for $n_0 \leq n$. It means we pad the kernel $K$ to match the spatial dimension $n_0 \times n_0$, instead of $n \times n$. To compensate for the error committed by the sub-sampling approximation, we multiply the bound by a factor $\alpha$. 
The work of \citet{pfister2019bounding} analyzes the quality of approximation depending on $n_0$ and gives an expression for factor $\alpha= \frac{2\left\lfloor k/2 \right \rfloor}{n_0}$ to compensate and ensure to remain an upper-bound, as studied in \cite{araujo2021lipschitz}.


\subsection{Approximation between circular and zero paddings}
\label{section:approx_circ}

In CNN architectures, zero padding is more commonly used than circular padding. To evaluate the potential approximation gap between these two methods, we compared our approach with the zero padding approach of \cite{Ryu19Plug} in terms of error. As shown in Figure~\ref{fig:times_estimation_error_vary_cin_cout} and Figure~\ref{fig:times_estimation_error_vary_input_size} of our paper, our results demonstrate that the error with respect to the zero padding case remains low. However, for high kernel sizes, as depicted in Figure~\ref{fig:times_estimation_error_vary_kernel_size}, our approach may not be as accurate as that of \cite{Ryu19Plug}. It is important to note that the typical kernel size for deep learning applications is relatively small, usually lower than 7.

An empirical comparison between zero and circular padding can be obtained by examining the table of Lipschitz constant estimations for all convolutional layers of ResNet18 in Table~\ref{tab:lipschitz_resnet18} of Appendix~\ref{appendix:table_convolutinonal_lip}.
We compare the values given by both methods: zero padding \cite{Ryu19Plug} used as reference and, circular padding \cite{sedghi2019singular}/Ours. We observe that the Lipschitz constant values obtained from both methods are quite similar. This provides a strong assessment of the approximation gap between circular and zero padding for convolutional layers encountered in ResNet.

\cite{yi2020asymptotic} provides an asymptotic bound on the singular values of circular convolution and zero padding convolutions in relation to the input spatial size. This, along with the findings from \cite{araujo2021lipschitz}, can aid in establishing an upper bound on the spectral norm of Toeplitz matrices, which represents zero padding convolutions, by leveraging the spectral norm of circulant matrices, which represents circular padding ones. Based on empirical results, the difference in the Lipschitz constant between circular and zero padding increases with kernel size and decreases with input size. Specifically, Figure~\ref{fig:times_estimation_error_vary_input_size} shows that as the input spatial size increases, the gap between \cite{Ryu19Plug} and \cite{sedghi2019singular}/our approach decreases, while the gap widens as kernel size increases.


\newpage

\section{Additional figures and tables for computation of spectral norms on GPU}

\begin{figure}[h]
    \centering
     \includegraphics[width=0.6\textwidth]{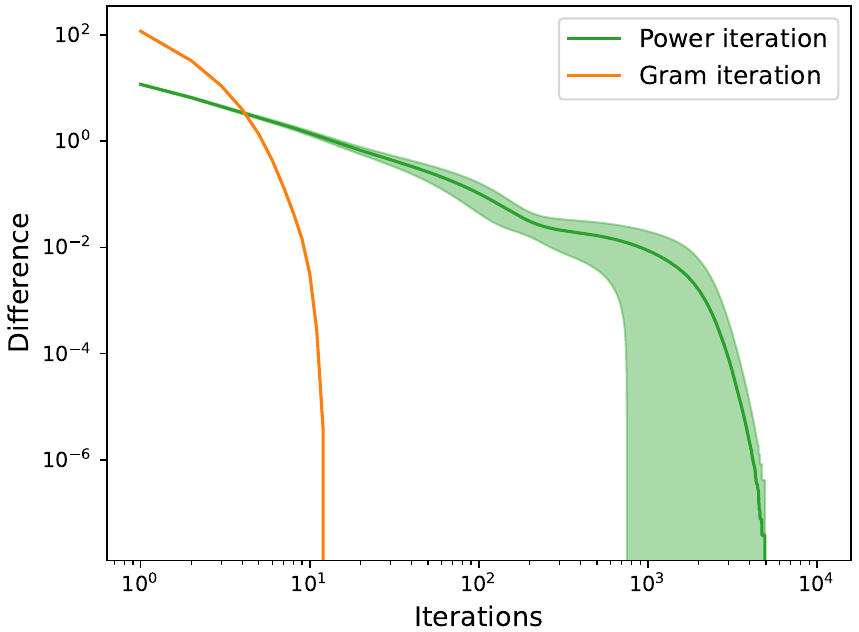}
     \caption{Convergence plot in log-log scale for spectral norm computation, comparing Power iteration and Gram iteration, one standard deviation shell is represented in a light color.
    Difference at each iteration is defined as $| {\sigma_1}_{\text{method}} - {\sigma_1}_{\text{ref}} |$ where reference is taken from PyTorch.
    Gram iteration converges to numerical 0 in less than 12 iterations while Power iteration requires up to 5,000.}
\label{fig:convergence_plot_gram_vs_power_itreration_5000_iters}
\end{figure}

\begin{figure}[h!]
    \centering
     \includegraphics[width=0.6\textwidth]{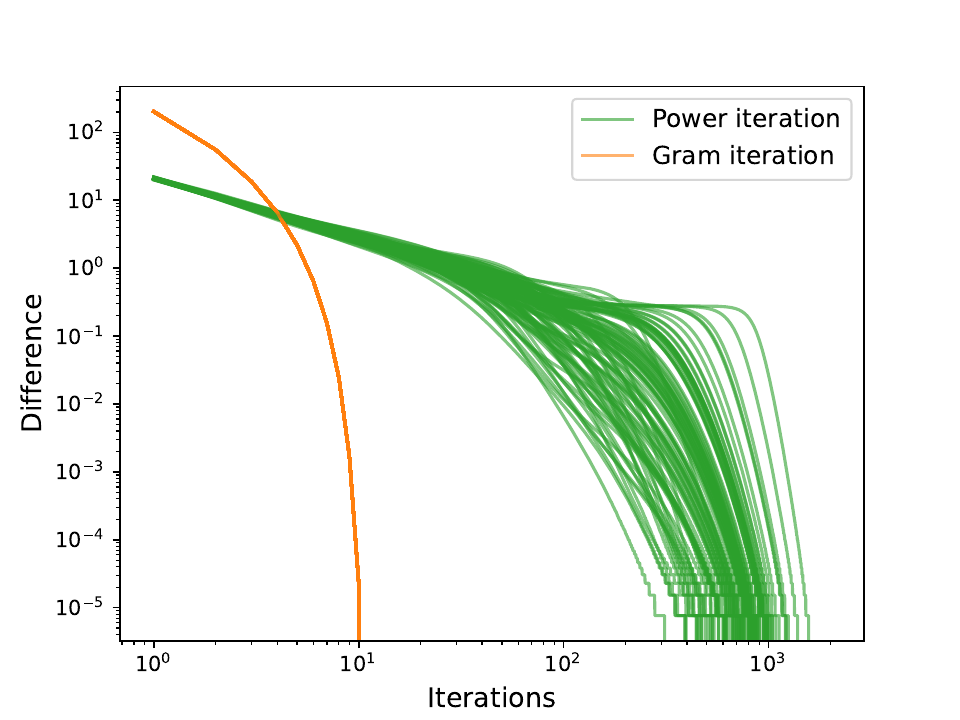}
    \caption{
    Convergence plot in log-log scale for spectral norm computation, comparing power iteration and Gram iteration, the same matrix is used here, each line corresponds to one run of Power iteration and Gram iteration. 
    The difference is defined as 
    $|{\sigma_1}_{\text{method}} - {\sigma_1}_{\text{ref}}|$ where reference is taken from PyTorch.
    We see that Power iteration has inherent randomness by design whereas Gram iteration is fully deterministic. 
    }
\label{fig:convergence_plot_gram_vs_power_itreration_same_mat}
\end{figure}

\begin{table}[h]
\caption{Numerical values associated to Figure~\ref{fig:comparaison_methods_compute_spectral_norm_matrix}: error ratios and computational times of methods for spectral norm computation. Error ratio is defined as ${\sigma_1}_{\text{method}} / {\sigma_1}_{\text{ref}} - 1$.}
\centering
\footnotesize
\begin{tabular}{lrrr}\toprule
\textbf{Method} &\textbf{Error ratio} &\textbf{Computational time (s)} \\\midrule
Reference & 0 &1.47e+0 \\
PI (10 iters) &-3.27e-2 $\pm$ 3.55e\text{-}3 &1.12e-3 \\
PI (100 iters) &-1.41e-3 $\pm$ 1.75e\text{-}3 &1.03e-2 \\
PI (1000 iters) &-2.43e-6 $\pm$ 4.01e\text{-}6 &1.04e-1 \\
Ours GI (10 iters) &6.75e-6 &1.83e-3 \\
Ours GI (11 iters) &4.73e-8 &1.99e-3 \\
Ours GI (12 iters) &4.33e-12 &2.13e-3 \\
\bottomrule
\end{tabular}
\label{tab:comparaison_methods_compute_spectral_norm_matrix}
\end{table}

\begin{table}[!htp]\centering
\caption{Times and difference to reference,
for different $p \times p$ matrices, averaged 10 times. Error is defined as $\sigma_{\text{method}} - \sigma_{\text{ref}}$.  To have comparable time and precision, $100$ iterations were taken for Power iteration and $10$ for Gram iteration. Gram iteration is overall more competitive than Power iteration for all matrix dimensions considered.
}
\footnotesize
\begin{tabular}{lrrrrrr}
\toprule
\textbf{Dimension p} & \textbf{PI time (s)} & \textbf{GI time (s)} & \textbf{PI error} & \textbf{GI error} & \textbf{Ref} \\
\midrule
10 &0.0093 &0.0012 &0.0006 &0 &5.28 \\
100 &0.0083 &0.0010 &-0.0231 &0 &19.31 \\
1000 &0.0093 &0.0017 &-0.0426 &0.0005 &62.95 \\
10000 &0.016 &0.0019 &-0.6004 &0.0037 &199.77 \\
50000 &0.017 &0.0040 &-1.70 &0.1 &447.21 \\
\bottomrule
\end{tabular}
\label{tab:comparaison_methods_compute_spectral_norm_dimension}
\end{table}


\newpage

\section{Lipschitz constant of CNN}
\label{section:lip_cnn}

Works of \cite{tsuzuku2018lipschitz, cisse2017parseval} studied  the Lipschitz constant of neural network by proposing for each kind of layer a formula or method to bound its Lipschitz constant and then  the network overall Lipschitz constant using Equation~(\ref{eq:prod_lip_bound}).

A dense layer can be represented as matrix $W \in \mathbb{R}^{m \times n}$. The Lipschitz constant for the $\ell_2$-norm equals to $\sigma_1(W)$. One can use SVD, power iteration, or Gram iteration to compute it.
Most of the activation functions used in deep learning such as ReLu, sigmoid, and hyperbolic tangent are $1$-Lipschitz.
As mentioned in Appendix D.3 of \cite{tsuzuku2018lipschitz} the max-pooling can be seen as a max operation followed by a pooling operation.
The Lipschitz constant of the max function is bounded by $1$, hence bound on the Lipschitz constant of max-pooling is bounded by the Lipschitz constant of the pooling operation. The Lipschitz constant of the pooling operation is bounded by the maximum number of repetitions of a pixel in the input. For a kernel of spatial size $k$ the number of maximum repetitions is $k^2$, taking into account the spatial size $n$ of the input and the stride the Lipschitz constant is bounded by \cite{tsuzuku2018lipschitz}:
\begin{equation*}
\left\lceil \frac{\min{(k, n-k+1)}}{\text{stride}} \right\rceil^2 \ .
\end{equation*}
Lipschitz constant of batch normalization is estimated in Appendix C.4.1 of \cite{tsuzuku2018lipschitz}: 
\begin{equation*}
L=\max_i \{ \left| \gamma_i \right| / \sqrt{\sigma_i^2 + \epsilon} \} \ ,
\end{equation*}
where, in this equation, $\gamma_i$ and $\sigma_i^2$ are respectively the scale parameter and the variance of the batch.
This Lipschitz constant is rarely considered in the global constant of networks.

Residual connections are present in ResNet \cite{he2016deep} architectures. In order to compute a bound on the Lipschitz constant of such networks, one can use the following rules as in \cite{tsuzuku2018lipschitz} and \cite{gouk2021regularisation}. For $g, f$ functions of Lipschitz constant of $L_1, L_2$ respectively, the Lipschitz constant for an addition $f + g$ is bounded by $L_1 + L_2$, and for composition $f \circ g$, it is bounded by $L_1 L_2$.
We call $\Phi_l$ the $l$-th composition of layers separated by residual connections, \ie $\Phi_l = \phi_{l, 1} \circ \phi_{l, 2} \circ \dots$ and 
$\Lip(\Phi_l) \leq \Lip(\phi_{l, 1}) ~\Lip(\phi_{l, 2}) \dots$.
We call $L_l$ the Lipschitz constant of the truncated network before block $\Phi_l$,
then we have $L_{l+1} \leq L_l + \Lip{(\Phi_l)}$.


\newpage

\section{Additional table for estimation on convolutional layers of CNNs}
\label{appendix:table_convolutinonal_lip}
\begin{table*}[h]
\caption{Comparison of Lipschitz bounds for all convolutional layers ResNet18, reference for real Lipschitz bound is given by \cite{Ryu19Plug} method. We observe that our method gives close value to \citet{sedghi2019singular} up to two decimal digits while being significantly faster. We remark that convolution filter $512 \times 512 \times 1 \times 1$ \citet{singla2021fantastic} value $2.03$ underestimates reference value $2.05$.}
\centering
\scriptsize
\begin{tabular}{lllllllllllll}\toprule
\textbf{Filter Shape} &\multicolumn{5}{c}{\textbf{Lipschitz Bound}} & &\multicolumn{5}{c}{\textbf{Running Time (s)}} \\\cmidrule{2-6}\cmidrule{8-12}
&Ours &Singla2021 &Araujo2021 &Sedghi2019 &Ryu2019 & &Ours &Singla2021 &Araujo2021 &Sedghi2019 &Ryu2019 \\\cmidrule{2-6}\cmidrule{8-12}
64 $\times$ 3 $\times$ 7 $\times$ 7 &15.91 &28.89 &22.25 &15.91 &15.91 & &0.0024 &0.0507 &0.0003 &18.06 &0.7579 \\\cmidrule{1-6}\cmidrule{8-12}
64 $\times$ 64 $\times$ 3 $\times$ 3 & 6.00 &9.32 &17.59 &6.00 &6.00 & &0.0014 &0.0264 &0.0003 &7.47 &1.07 \\\cmidrule{1-6}\cmidrule{8-12}
64 $\times$ 64 $\times$ 3 $\times$ 3 &5.34 &6.28 &15.31 &5.34 &5.33 & &0.0015 &0.025 &0.0012 &7.29 &1.07 \\\cmidrule{1-6}\cmidrule{8-12}
64 $\times$ 64 $\times$ 3 $\times$ 3 &7.00 &8.71 &16.46 &7.00 &7 & &0.0011 &0.026 &0.0003 &7.27 &1.08 \\\cmidrule{1-6}\cmidrule{8-12}
64 $\times$ 64 $\times$ 3 $\times$ 3 &3.82 &5.40 & 13.74 &3.81 &3.82 & &0.0014 &0.025 &0.0003 &7.19 &1.08 \\\cmidrule{1-6}\cmidrule{8-12}
128 $\times$ 64 $\times$ 3 $\times$ 3 & 4.71 &5.98 &16.35 &4.71 &4.71 & &0.0014 &0.026 & 0.0012 & 9 & 0.473 \\\cmidrule{1-6}\cmidrule{8-12}
128 $\times$ 128 $\times$ 3 $\times$ 3 &5.72 &7.21 & 23.58 &5.72 &5.72 & &0.0011 &0.026 &0.0013 &4.85 &0.78 \\\cmidrule{1-6}\cmidrule{8-12}
128 $\times$ 64$\times$ 1 $\times$ 1 &1.91 &1.91 &6.39 &1.91 &1.91 & &0.0014 & 0.0333 &0.0001 &9.25 &0.042 \\\cmidrule{1-6}\cmidrule{8-12}
128 $\times$ 128 $\times$ 3 $\times$ 3 &4.39 &6.78 & 21.93 &4.39 &4.41 & &0.0011 &0.026 &0.0003 &4.844 &0.78 \\\cmidrule{1-6}\cmidrule{8-12}
128 $\times$ 128 $\times$ 3 $\times$ 3 &4.89 &7.57 & 19.79 &4.89 & 4.88 & & 0.0014 & 0.027 &0.0004 &4.78 & 0.78 \\\cmidrule{1-6}\cmidrule{8-12}
256 $\times$ 128 $\times$ 3 $\times$ 3 &7.39 &8.45 &35.86 & 7.39 &7.39 & &0.0014 &0.026 &0.0017 &5.36 & 0.46 \\\cmidrule{1-6}\cmidrule{8-12}
256 $\times$ 256 $\times$ 3 $\times$ 3 &6.58 &8.03 &37.98 &6.58 &6.58 & &0.0014 &0.026 &0.0025 &3.46 &0.85 \\\cmidrule{1-6}\cmidrule{8-12}
256 $\times$ 128$\times$ 1 $\times$ 1 &1.21 &1.21 &5.97 &1.21 & 1.21 & &0.0014 & 0.036 &0.0001 &5.96 & 0.04 \\\cmidrule{1-6}\cmidrule{8-12}
256 $\times$ 256 $\times$ 3 $\times$ 3 &6.36 &7.57 &30.07 &6.36 &6.35 & &0.001 &0.026 &0.0004 &3.48 &0.85 \\\cmidrule{1-6}\cmidrule{8-12}
256 $\times$ 256 $\times$ 3 $\times$ 3 & 7.68 &9.17 &29.19 & 7.68 &7.67 & & 0.0011 &0.027 &0.0004 &3.48 & 0.85 \\\cmidrule{1-6}\cmidrule{8-12}
512 $\times$ 256 $\times$ 3 $\times$ 3 &9.99 &11 &46.24 &9.99 &9.92 & &0.0011 &0.026 &0.004 &3.55 &0.45 \\\cmidrule{1-6}\cmidrule{8-12}
512 $\times$ 512 $\times$ 3 $\times$ 3 &9.09 &10.45 &47.59 &9.09 &9.04 & &0.0014 &0.026 &0.0001 &2.8 &0.79 \\\cmidrule{1-6}\cmidrule{8-12}
512 $\times$ 256 $\times$ 1 $\times$ 1 &2.05 &2.03 &11.87 &2.05 &2.05 & &0.0012 &0.024 &0.0001 &3.77 &0.052 \\\cmidrule{1-6}\cmidrule{8-12}
512 $\times$ 512 $\times$ 3 $\times$ 3 &17.60 &18.37 &59.04 &17.60 &17.5 & &0.001 &0.027 &0.0005 &2.79 &0.79 \\\cmidrule{1-6}\cmidrule{8-12}
512 $\times$ 512 $\times$ 3 $\times$ 3 &7.48 & 7.6 & 57.33 &7.48 &7.43 & &0.0014 &0.027 &0.0006 &2.83 &0.79 \\
\bottomrule
\end{tabular}
\label{tab:lipschitz_resnet18}
\end{table*}


\section{Additional table and figure for regularization of convolutional layers of ResNet}

The naive auto differentiation indeed uses a lot amount of GPU memory, but the explicit differentiation using Proposition~\ref{prop:gradient_gram_iteration} mitigates this matter as the  computational graph for the gradient is smaller, as reported in Table~\ref{tab:mem_gpu}.

\begin{table}[!htp]
\centering
\caption{Differentiation memory footprints in order to perform  regularization for all convolutional layers in the model, for an input image of size $224 \times 224$. The number of Gram iterations is 6.}
\footnotesize
\begin{tabular}{lrrr}\toprule
Model &Autodiff Mem GPU (MB) &Explicit Mem GPU (MB) \\\midrule
ResNet18 &9770 &4322 \\\midrule
ResNet34 &17152 &7088 \\ \midrule
ResNet50 &< 42 000 &27535 \\
\bottomrule
\end{tabular}
\label{tab:mem_gpu}
\end{table}

\begin{figure}[h]
\includegraphics[width=0.95\textwidth, height=0.3\textwidth]{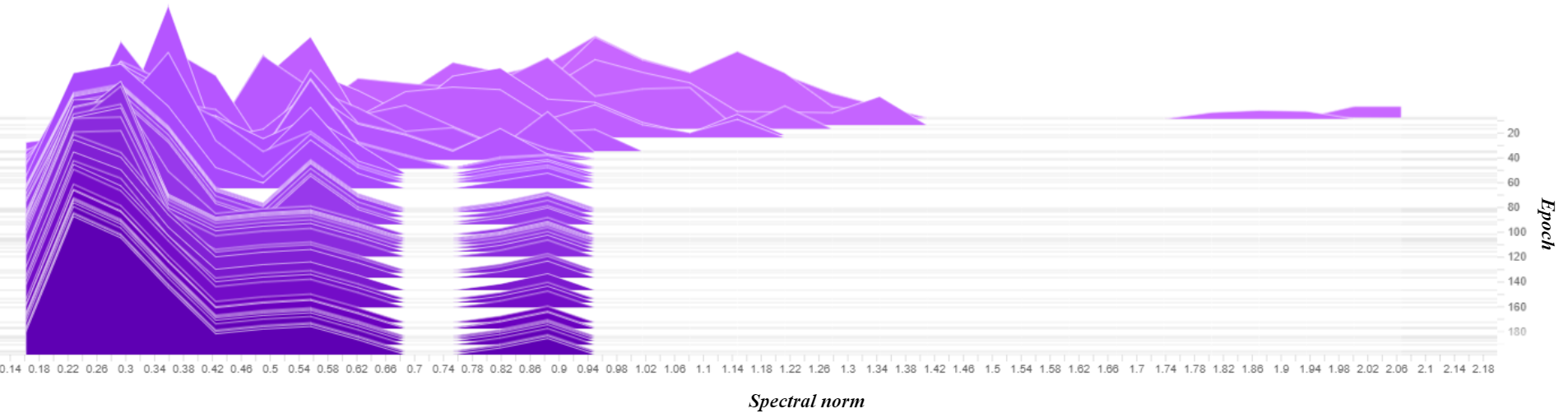}
    \caption{Spectral norm histograms in the function of epoch for  training on CIFAR10 with regularization using Araujo2021.}
    \includegraphics[width=0.7\textwidth, height=0.3\textwidth]{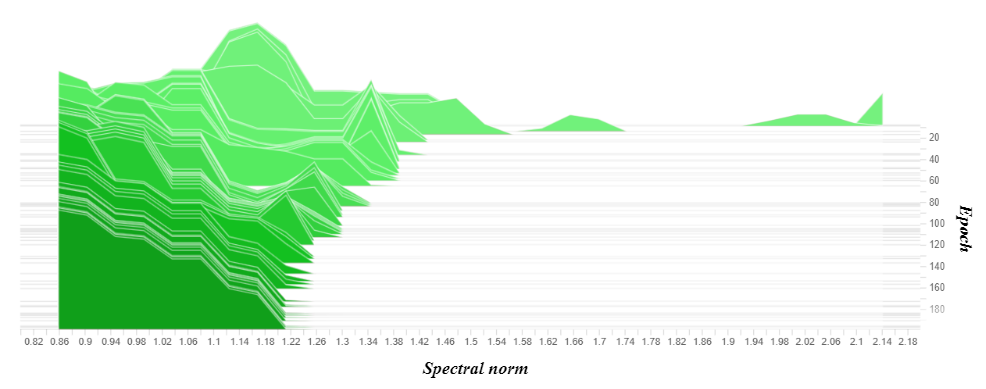}
    \centering
    \caption{Spectral norm histograms in the function of epoch for  training on CIFAR10 with regularization using Singla2021.}
    \includegraphics[width=0.7\textwidth, height=0.3\textwidth]{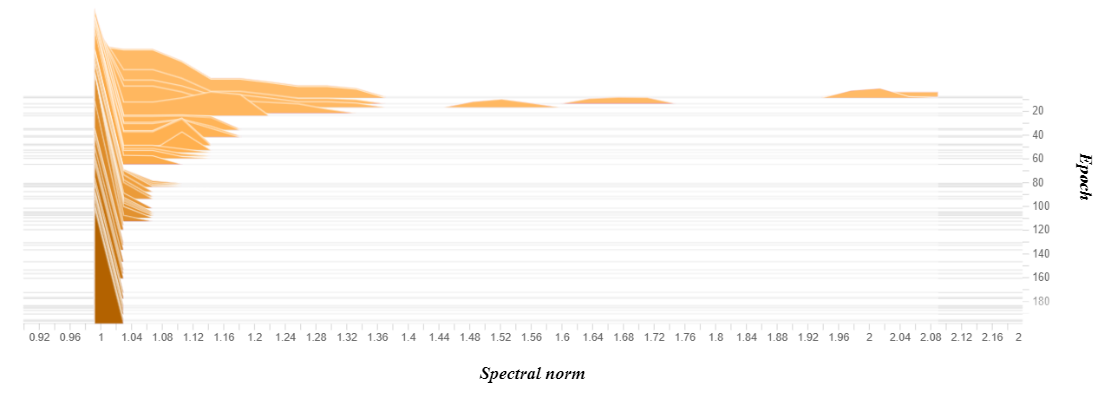}
    \caption{Spectral norm histograms in the function of epoch for  training on CIFAR10 with regularization using our method.}
    \label{fig:spectral_norm_over_epochs_reg_resnet_training_for_bounds}
\end{figure}

\end{document}